%% file: main.tex
\documentclass[letterpaper]{article} 
\usepackage{aaai2026}
\usepackage{times}  
\usepackage{helvet}  
\usepackage{courier}  
\usepackage[hyphens]{url}  
\usepackage{graphicx} 
\usepackage{natbib}  
\usepackage{caption} 
\usepackage{algorithm}
\usepackage{algorithmic}

\usepackage{newfloat}
\usepackage{listings}
\DeclareCaptionStyle{ruled}{labelfont=normalfont,labelsep=colon,strut=off} 
\floatstyle{ruled}
\newfloat{listing}{tb}{lst}{}
\floatname{listing}{Listing}
\title{Aggregate-Combine-Readout GNNs \\ Are More Expressive Than Logic $C^2$}
\author {
    Stan P Hauke\textsuperscript{\rm 1},
    Przemysław Andrzej Wałęga\textsuperscript{\rm 2}
}
\affiliations {
    \textsuperscript{\rm 1}King’s College London, UK\\
    \textsuperscript{\rm 2}Queen Mary University of London, UK\\
stanislaw.hauke@kcl.ac.uk, p.walega@qmul.ac.uk
}

\usepackage{bibentry}

\input{macros.tex} 

\begin{document}

\maketitle

\begin{abstract}
In recent years, there has been growing interest in understanding the expressive power of graph neural networks (GNNs) by relating them to logical languages. This research has been been initialised by an influential result of  \citet{DBLP:conf/iclr/BarceloKM0RS20}, who showed that the graded modal logic (or a guarded fragment of the logic $\Ctwo$), characterises the logical expressiveness of aggregate-combine GNNs. As a ``challenging open problem'' they left the question whether full $\Ctwo$ characterises the logical expressiveness of aggregate-combine-readout GNNs.
This question has remained unresolved despite several attempts. 

In this paper, we solve the above open problem by proving that the logical expressiveness of aggregate-combine-readout GNNs strictly exceeds that of $\Ctwo$. This result holds over both undirected and directed graphs. Beyond its implications for GNNs, our work also leads to purely logical insights on the expressive power of infinitary logics. 
\end{abstract}

\section{Introduction}

\emph{Graph Neural Networks}  (GNNs) \cite{DBLP:conf/icml/GilmerSRVD17}  are state-of-the-art machine learning models tailored for processing graph structured data.
They have been successfully applied across numerous domains, including molecular property prediction \cite{DBLP:journals/air/BesharatifardV24}, traffic forecasting and navigation \cite{Derrow_Pinion_2021}, visual scene interpretation \cite{DBLP:journals/corr/abs-2209-13232},  personalised recommendations \cite{Ying_2018}, and
knowledge graph completion and reasoning under partial information \cite{DBLP:conf/iclr/CucalaGKM22,DBLP:conf/nips/ZhangC18,DBLP:journals/corr/abs-2402-04062}.

In recent years, there has been growing interest in understanding the expressive power of GNNs, particularly focusing on the basic message-passing architecture.
A key result in this area \cite{DBLP:conf/aaai/0001RFHLRG19,DBLP:conf/iclr/XuHLJ19} shows that GNNs have the same \emph{distinguishing power} as the Weisfeiler–Leman (WL) algorithm \cite{weisfeiler1968reduction}—a widely used heuristic for testing graph isomorphism \cite{DBLP:conf/focs/BabaiK79}.
This means that a pair of graphs can be distinguished by WL if and only if there exists a GNN that can distinguish them.
In turn, a classical result by \citet{DBLP:journals/combinatorica/CaiFI92} shows that WL has the same distinguishing power as the fragment $\Ctwo$ of first-order logic (\FO), in which formulas are restricted to two variables but may use counting quantifiers $\exists_k$, interpreted as ``there exist at least $k$ distinct elements such that~$\dots$.''
As a result, we obtain a tight correspondence between GNNs, the Weisfeiler–Leman algorithm, and the logic $\Ctwo$.

The results on the distinguishable power, however, do not allow us to establish a one-to-one mapping between GNNs and logical formulas expressing the same properties.
This finer correspondence is known as \emph{logical expressiveness}, or \emph{uniform expressive power}, and has attracted growing interest in recent years \cite{DBLP:conf/icalp/BenediktLMT24,ahvonen2025logicalcharacterizationsrecurrentgraph,schonherr2025logical,ijcai2024p391,DBLP:conf/kr/CucalaG24}.
The main, and historically first, results in this direction have been established by
\citet{DBLP:conf/iclr/BarceloKM0RS20}.
They have studied two architectures of message-passing GNNs: the standard  aggregate-combine GNNs (AC-GNNs) and their extension with readout function called aggregate-combine-readout GNNs (ACR-GNNs).
The two main results of \citet{DBLP:conf/iclr/BarceloKM0RS20} are as follows:
\begin{enumerate}[label=(\roman*)]
\item 
the \FO{} node properties expressible by AC-GNNs are exactly those definable in graded modal logic,
\item 
the \FO{} node properties expressible by ACR-GNNs contain all properties definable in $\Ctwo$.
\end{enumerate}
Note that, in contrast to Result (i), Result (ii) does not provide an exact logical characterisation. This was left by the authors' as ``a challenging open problem.''

\paragraph{The Open Problem}
The precise formulation of the  open problem introduced  by  \citet{DBLP:conf/iclr/BarceloKM0RS20} is whether 
the \FO{} node properties expressible by ACR-GNNs are exactly those definable in $\Ctwo$.
Notably, although their paper explicitly states this problem and has since become widely known and frequently cited, the question has remained unresolved for the past five years.
To the best of our knowledge, several research groups have attempted to solve this problem \cite{DBLP:conf/aaai/PfluegerCK24,DBLP:conf/icalp/BenediktLMT24}, but without success.

\paragraph{Contributions} 
In this paper we will solve the above open problem, by showing that ACR-GNNs can express \FO{} node classifiers beyond $\Ctwo$.

We will show that this result holds not only in the setting of undirected graphs---as originally considered by \citet{DBLP:conf/iclr/BarceloKM0RS20}---but also in the setting of directed graphs.
In both cases, our proofs follow a common structure: (1) we define a node property, (2) we show that it is expressible both in \FO{} and by an ACR-GNN, and (3)  we show that the property is not expressible in $\Ctwo$.
In the directed case (\Cref{sec:directed}), the property we consider is that of ``being a node of a graph whose edge relation forms a strict linear order.'' In the undirected case (\Cref{sec:undirected}),  we simulate directed edges using paths of three undirected edges, where  direction of an edge is encoded by colours of the two middle nodes in the path.
We then consider the property of being a node of an undirected graph that simulates a strict linear order.

To show Results (1) and (2) we provide explicit constructions of \FO{} formulas and ACR-GNNs, respectively.
To show the inexpressibility Results (3), we introduce in \Cref{sec:WL} a bounded version of WL algorithm, which characterises expressive power of \Ctwo{} formulas with counting quantifiers $\exists_k$ mentioning bounded $k$ only.
Using this characterisation, we prove inexpressibility results for both directed and undirected settings.

Finally, in \Cref{sec:impact} we will exploit our results to the study the expressive power of infinitary logics.
As we show, the infinitary version of \Ctwo{} can express strictly more \FO{} properties than the standard, finitary \Ctwo.

\section{Preliminaries}

In this section we will introduce basic notions and notation for graphs, GNNs, and logics.
We will extend the setting of \citet{DBLP:conf/iclr/BarceloKM0RS20}, by considering not only undirected, but also directed graphs.

\paragraph{Graphs}
A directed (node-labelled, finite, and simple) \emph{graph} of dimension $d \in \mathbb{N}$ is a tuple   $\G = (V, E, \lambda)$, where $V$ is a finite set of nodes, $E \subseteq V \times V$ is a set of directed edges with no loops $E(v,v)$, and $\lambda:V \to \{0,1 \}^d$ assigns to each node a 
binary vector of a dimension $d$.
We will identify \emph{undirected} graphs with directed graphs that have a symmetric edge relation, and write $\{v,w\}$ for a pair of edges $(v,w)$, $(w,v)$.
The \emph{neighbourhoud}, $N_\G(v)$, of a node $v$ in a graph $G$,  is the  set of all nodes $w$ 
such that $G$ has an edge (in any direction)
between $w$ and $v$.
The \emph{in-neighbourhoud}, $\Ni_\G(v)$, 
are nodes $w$ 
such that $G$ has an edge
from $w$ to $v$, whereas
the \emph{out-neighbourhoud}, $\No_\G(v)$, are nodes $w$ 
such that $G$ has an edge
from $v$ to $w$.
Hence, in undirected graphs we have $N_\G(v) = \Ni_\G(v) = \No_\G(v)$.

\paragraph{GNN Node Classifiers}
We focus on  \emph{aggregate-combine-readout} GNNs (ACR-GNNs) introduced by \citet{DBLP:conf/iclr/BarceloKM0RS20}, which extend the standard message-passing mechanism with readout functions.
First, we introduce  ACR-GNN architecture for processing undirected graphs.
In such GNNs, each layer is a triple $( \agg, \comb, \readout )$ consisting of  an \emph{aggregation function}, $\agg$, mapping a multiset (a generalisation of a set so that elements can have multiple occurrences) of vectors into a single vector,
a \emph{combination function} $\comb$, mapping a vector to a vector, and a \emph{readout function},  $\readout$,  mapping a multiset  of vectors into a single vector.
Such layers applied to a graph $\G = ( V, E, \lambda )$ computes a graph $\G' = ( V, E, \lambda' )$ with a new  labelling function 
$\lambda'$ such that for each $v$, the labelling $\lambda'(v)$ is given by
$$
\comb \Big( \lambda(v), \agg( \lBrace  \lambda(w) \rBrace_{w \in N_G(v)} ), \readout( \lBrace  \lambda(w) \rBrace_{w \in V}) \Big), 
$$
where $\lBrace \cdot \rBrace$ stands for a multiset. 
In the spirit of \citet{DBLP:conf/log/RossiCGFGB23}, we also consider a straightforward generalisation of ACR-GNN architecture for  processing directed graphs.
In this case a GNN is a tuple $( \aggi, \aggo, \comb, \readout )$, which has two types of aggregation: $\aggi$ for incoming edges and \aggo{} for outgoing edges.
In such ACR-GNNs, a new labelling $\lambda'(v)$ is computes as
\begin{align*}
\comb \Big( & \lambda(v), \aggi( \lBrace  \lambda(w) \rBrace_{w \in \Ni_G(v)} ),
\\
&
\aggo( \lBrace  \lambda(w) \rBrace_{w \in \No_G(v)} ),
\readout( \lBrace  \lambda(w) \rBrace_{w \in V}) \Big).
\end{align*}
An \emph{ACR-GNN classifier} $\N$ of dimension $d$ consists of a fixed number $L$ of layers\footnote{We assume that functions in the layers are of matching dimensions, so that they can be applied.} and a 
classification function $\mathsf{cls}$ from vectors to truth values; once applied to a graph of dimension $d$, the classifier $\N$ computes for each node $v$  a truth value $\N(G,v)$.

\paragraph{Logical Node Classifiers}
In this paper, by \FO{} we mean  the standard first-order logic 
with identity $=$, one binary predicate $E$ for edges, and unary predicates $P_1, \dots, P_d$ for node labels.
We will consider also the fragment \Ctwo{}  of \FO{}, which allows for using only two variables in formulas, but allows for additional counting quantifiers $\exists_k$, for any $k \in \mathbb{N}$, where $\exists_k x \varphi(x)$ means that $\varphi$ holds in at least $k$ different nodes.
We will  write $\exists_{=k}\vp(x)$ as an abbreviation for
$\exists_{k}\vp(x)\land \neg\exists_{k+1}\vp(x)$.
Note that we write $\varphi(x)$ for a formula with exactly one free variable $x$, and similarly we will use $\varphi(x,y)$ for a formula with exactly two free variables.
We let the quantifier \emph{depth} of a formula $\varphi$ be
its maximum nesting of quantifiers.
Moreover, for $\Ctwo$ formulas we define the  \emph{counting rank},  $\mathsf{rk}_{\#}(\varphi)$, as
the maximal  among numbers $k$ occurring in its counting quantifiers.
For a logic $\L$, we let $\L_{\ell,c}$ be the fragment allowing for formulas of depth at most $\ell$ and counting rank at most $c$;
our paper will pay special attention to $\Ctwo_{\ell,c}$.

A \emph{logical node classifier} is  a formula $\varphi(x)$ in \FO{} (or its fragment)  with one free variable.
To evaluate logical classifiers, we identify a graph $\G = (V, E, \lambda)$ of dimension $d$ with the corresponding \FO{} structures
$\M_{\G} = (V, P_1,  \dots, P_d, E)$, with domain $V$, sets  $P_i = \{ v \in V \mid \lambda(v)_i=1 \}$ containing all nodes $v$ with $1$ on the $i$th position of $\lambda(v)$, and the binary relation
$E$ being the graph edges.
We  assume the standard \FO{} semantics over such models and write $G \models \varphi(v)$ if classifier $\varphi(x)$ holds in $\M_{\G}$ at the node $v$.
If this is the case, we say that 
the application of the logical classifier $\varphi(x)$ to  $\G$  at node $v$ is $\true$, and otherwise it is $\false$.
We write $G,u \equiv_\L H,v$, if 
$G \models \varphi(u)$ is equivalent to $H \models \varphi(v)$, for each logical classifier $\varphi(x)$ in a logic $\L$.

\section{WL Algorithm with Bounded Counting}\label{sec:WL}

In this section, we will introduce a bounded version of the one dimensional  \WL{} algorithm \cite{weisfeiler1968reduction}. 
Our version $\WL_{c}$  is parametrised by  $c \in \mathbb{N}$, which bounds the ``counting abilities'' of the algorithm.
As we will show, $\ell$ rounds of application of $\WL_c$  allows us to characterise expressiveness of the fragment $\Ctwo_{\ell,c}$ of $\Ctwo$, where formulas have depth bounded by $\ell$ and counting rank by $c$. 
This result will play a crucial role to establish non-expressivity results in the latter sections of the paper.

The main idea behind $\WL_{c}$  is that  the algorithm is insensitive to multiplicities (occurring in  processed multisets) greater than $c$.
In particular, instead of computing new labels for nodes based on multisets $\lBrace \cdot \rBrace$ of labels, the computations are based on the \emph{$c$-bounded} multisets $\lBrace \cdot \rBrace^c$, obtained by reducing all multiplicities to at most $c$.
For example $\lBrace 7,7,7,3 \rBrace^2 = \lBrace 7,7,3 \rBrace$.
In particular, over undirected graphs, labelling $W^{\ell+1}_c(v)$ of a node $v$ in iteration $\ell+1$ will depend on the the previous label $W^{\ell}_c(v)$, the $c$-bounded multiset of labels of $v$ neighbours, and  the $c$-bounded multiset of non-neighbours, so $W^{\ell+1}_c(v) $ equals 
\begin{align*}
( 
W^\ell_c(v), 
\lBrace 
W^\ell_c(w)
\rBrace_{w \in N_G(v)}^c,
\lBrace W^\ell_c(w) \rBrace^c_{w \in V \setminus \{N_G(v) \cup \{v\}\}} 
).
\end{align*}
Characterising  $\Ctwo_{\ell,c}$ over directed graphs is more challenging.
In this case, instead of 
considering in $\WL_c$ one multiset of neighbours' labels, we  consider separately nodes which belong to $\Ni_G$ and $\No_G$, those which belong to $\Ni_G$ only, and those which belong to $\No_G$ only.
Below we define the algorithm, which in the case of undirected graphs reduces to the computations presented above.

\begin{definition}\label{defdef1}
Let $c \in \mathbb{N}$.
The  \emph{$c$-bounded \WL{} algorithm}, $\WL_c$, takes as an input a graph $\G=(V,E,\lambda)$,
and computes labels $W_c^\ell(v)$ for all $v \in V$ as follows:
\begin{align}
W^{0}_c(v) & = \lambda(v) \notag \\
\label{eq:WLupdate}
W^{\ell+1}_c(v) & = 
\Big( 
 W^\ell_c(v), 
\lBrace W^\ell_c(w) \rBrace^c_{w \in \Ni_G(v) \cap \No_G(v)},\\
& \lBrace W^\ell_c(w) \rBrace^c_{w \in \Ni_G(v) \setminus \No_G(v)}, 
 \lBrace W^\ell_c(w) \rBrace^c_{w \in \No_G(v) \setminus \Ni_G(v)}, \notag \\
 &\lBrace W^\ell_c(w) \rBrace^c_{w \in V \setminus \{N_G(v) \cup \{v\}\}} 
\Big).  \notag 
\end{align}
\end{definition}

The above idea of considering various combinations of in- and out-neighbours is closely related to the introduction of complex modal operators, which was used by \citet{DBLP:conf/csl/LutzSW01} to construct a modal logic of the same expressiveness as \FOtwo{}.
It is also worth observing that, over undirected graphs, $\WL_c$ with $c < \infty$ is strictly less expressive than the standard $\WL$, whereas $\WL_c$ with $c = \infty$  would 
coincide exactly with  \WL{} \cite{DBLP:journals/combinatorica/CaiFI92}.
In \Cref{mainWLtheorem} we will show that $\WL_c$ characterises the expressiveness of $\Ctwo$  with  counting rank  $c$ over  directed (and so, also  undirected) graphs.
To obtain this result, we will use the following technical lemma, showing that over directed graphs, $\Ctwo$ formulas have a specific normal form.

\begin{restatable}{lemma}{normalform}\label{twovarstructure}
Over directed graphs, every  $\Ctwo_{\ell,c}$ formula is equivalent to a finite disjunction 
$$
\bigvee_{i=1}^n \big( \alpha_i(x) \land \beta_i(y) \land \gamma_i(x,y) \big),
$$ 
where
$\alpha_i(x), \beta_i(y) \in \Ctwo_{\ell,c}$
and each $\gamma_i(x,y)$ is one of the following five formulas:
$E(x, y) \land E(y, x)$,  $ E(x, y) \land \neg E(y, x)$, $\neg E(x, y) \land E(y, x)$, $\neg E(x, y) \land\neg  E(y , x)\land x\neq y$, and $x=y$.
\end{restatable}

\begin{proof}[Proof sketch]
Consider a $\Ctwo_{\ell,c}$ formula $\varphi(x, y)$.
With De Morgan and distributivity laws, we can transform $\varphi(x, y)$  into $\bigvee_{i=1}^n \bigwedge_{j=1}^{m_i} \psi_{i,j}$, where each $\psi_{i,j}$
is either a literal (an atom or its negation), or starts with $\exists_k$, or starts with  $\neg \exists_k$.
Next, we partition   each $\bigwedge_{j=1}^{m_i} \psi_{i,j}$  so that  we obtain
$\bigvee_{i=1}^n \big( \alpha_i(x) \land \beta_i(y) \land \gamma_i(x,y) \big)$.
Formulas $\alpha_i(x)$ and $\beta_i(y)$ are already in the required forms.
Next, we observe that  each $\gamma_i(x,y)$ is a conjunction of literals, as otherwise it would start with $\exists_k$ or $\neg \exists_k$, so it would not have two free variables.
The are six literals with two free variables, namely $E(x,y)$, $E(y,x)$,  $x=y$, and their negations.
Since we consider only simple graphs, we can show that 
each such $\gamma_i(x,y)$ can be written as a disjunction of  the five formulas  in the lemma.
Then, by applying distributivity laws, we can obtain the  form from the lemma.
\end{proof}

Next, we will use the normal form from \Cref{twovarstructure} to show that $\WL_c$ captures the expressive power of $\Ctwo$ with counting rank $c$.

\begin{restatable}{theorem}{WLtheorem}\label{mainWLtheorem}
Let $\ell , c \in \mathbb{N} $.
For any directed graphs $G$ and $H$ 
with nodes $u$ and $v$, the following holds:
\smallskip

\centerline{$ G,u \equiv_{\Ctwo_{\ell, c}} H,v$ \quad  if and only if \quad
$ W^\ell_c(u) = W^\ell_c(v)$.}
\end{restatable}
\begin{proof}[Proof sketch]
The proof is by induction on $i \leq  \ell$.
In the basis it suffices to observe that $ W^\ell_0(u) = W^\ell_0(v)$ means that $u$ and $v$ satisfy the same Boolean formulas with one free variable.
Next, we consider the inductive step.

For the forward implication we assume that $W^{i +1}_c(u) \neq W^{i +1}_c(v)$.
Hence  $W^{i +1}_c(u)$ and $W^{i +1}_c(v)$ need to differ on one of the five components from Equation~\eqref{eq:WLupdate}.
Assume that $\lBrace W^i_c(w) \rBrace^c_{w \in \Ni_G(u) \cap \No_G(u)} \neq \lBrace W^i_c(w) \rBrace^c_{w \in \Ni_H(v) \cap \No_H(v)}$ (the other cases are analogous).
So there are $k > k'$ both $\leq c$ such that some colour $t$ occurs $k$ times in the left multiset and $k'$ times in the right multiset.
Using the inductive hypothesis, we can show that there is a $\Ctwo_{{i},c}$ formula
$\psi_t^i(y)$ such that for any node $w$ in $F \in \{G,H \}$,  $W^i_c(w)=t$ if and only if $F\models \psi_t^i(w)$. 
Hence $G\models \exists_k y (\psi_t^i(y) \land E(u,y)\land E(y,u))$,
but $H \not\models \exists_k y (\psi_t^i(y) \land E(v,y)\land E(y,v))$.




For the backwards implication assume that \( W^{i+1}_c(u) = W^{i+1}_c(v) \).
We show by  induction on the structure of $\Ctwo_{{i+1},c}$ formulas  $\varphi(x)$ that
$
G \models \varphi(u)$ if and only if $H \models \varphi(v)$.
The interesting case is for $\varphi(x) = \exists_{k} y \psi(x, y)$.
We first show the above for the formulas $\psi(x,y)$ of the form $\eta(y)\land \gamma(x,y)$, where $\eta(y)\in \Ctwo_{l,c}$ and $\gamma(x,y)$ is one of the five formulas from  \Cref{twovarstructure}. We finish the proof by applying \Cref{twovarstructure} to lift this result to any $\psi(x,y) \in \Ctwo_{i,c} $.
\end{proof}


We will use \Cref{mainWLtheorem} in two following sections: in \Cref{sec:directed} for directed graphs (\Cref{thm:LonotC2}) and in \Cref{sec:undirected} for undirected graphs (\Cref{thm:GLonotC2}).

\section{Logical Expressiveness Over Directed Graphs}\label{sec:directed}

In this section, we will study the expressiveness of ACR-GNNs over directed graphs. 
In this setting, we will consider an analogous question to the open problem of \citet{DBLP:conf/iclr/BarceloKM0RS20}, namely: are $\Ctwo$ node classifiers exactly \FO{} classifiers expressible by ACR-GNNs?  
As we will show, and which may be surprising, the  answer is negative.
To this end, we will prove that checking if edges of a graph form a strict linear order is expressible in \FO{} and by ACR-GNNs, but cannot be expressed in $\Ctwo$. Although this is a property of graphs, we can formulate it also as a node classifier as follows.

\begin{definition}
We let $\varphi_\LO(x)$ be a node classifier accepting a node of a graph $G$ if and only if $G$ is a strict linear order.
\end{definition}
Clearly, strict linear orders can be defined in $\FO$ with a formula $\psi$ being a conjunction of the following three: 
\begin{align*}
&\forall x\, \neg E(x, x) 
&&\quad \text{irreflexivity} \\
&\forall x\, \forall y\, \Big( (x = y) \lor E(x, y) \lor E(y , x) \Big) 
&&\quad \text{totality} \\
&\forall x\, \forall y\, \forall z\, \Big( E(x , y) \land  E(y, z) \to E(x , z) \Big) 
&&\quad \text{transitivity}
\end{align*}
Since we are considering simple graphs, irreflexivity can be omitted from $\psi$. Notice that $\psi$ has no free variables, but we can always turn it into a node classifier by writing it as $(x=x) \land \psi$.
Thus, $\varphi_{\LO}(x)$ is expressible in $\FO$.

Next, we will show that  $\varphi_{\LO}(x) $ can be expressed as an ACR-GNN.
This is more challenging, since ACR-GNNs cannot detect  transitivity.
To address this challenge, we  will exploit the following equivalent definition of linear orders.

\begin{restatable}{proposition}{linearalternative}\label{linear}
A finite binary relation $E$ is a strict linear order if and only if $E$  is irreflexive, total, and each element has a different number of $E$-successors.
\end{restatable}
\begin{proof}[Proof sketch]
Strict linear orders clearly satisfy the three properties.
For the opposite direction we  show that $E$ enjoying these properties is transitive.
Assume that there are $n$ elements.
As each element has a different number of $E$-successors and $E$ is irreflexive, 
we can call the elements $v_0, \dots, v_{n-1}$, where $v_i$ is the unique element whose number of $E$-successors is  $i$.
By a strong induction on $i \leq n-1$, we can show that, for all  $v_j$, we have $(v_i,v_j) \in E$ if and only if $i > j$. It  implies that $E$ must be transitive.
\end{proof}

We will use \Cref{linear} to construct an ACR-GNN which detects strict linear orders.

\begin{restatable}{theorem}{GNNforLin}\label{LinInGNN}
Over directed graphs,
$\varphi_{\LO}(x)$ is expressible by an ACR-GNN.
It can be achieved using only 3 layers and
no aggregation over the out-neighbourhood.
\end{restatable}
\begin{proof}
We will construct the required ACR-GNN  $\N$, whose application to a linear order of length four is presented in \Cref{fig::GNN_LO}.
The first layer maps the initial vector of a node $v$ into the number $10^n$, where $n$ is the in-degree of $v$. This is obtained by setting  $\aggi(M)=  10^{|M|}$ and $\comb(x,y) =y$.
The second layer  maps a vector of  $v$ into a vector in $\R^2$ of the form $(10^n, 10^{k_1} + \dots + 10^{k_n} )$ where $10^n$ is as in the first layer, whereas each $k_i$ is the in-degree of the $i$th among the $n$ in-neighbours of $v$. This is obtained by setting  $\aggi(M)= sum (M)$ and $\comb(x,y) = (x,y)$.
The third layers  maps each vector into 1 or 0 by setting
$\readout(M) = 1$ if both of the following conditions hold:
\begin{enumerate}[label={}, leftmargin=2em]
\item[(i)] $x[1] \neq y[1]$, for every pair $x,y \in M$.
\item[(ii)] if $x[1] = 10^n$, then $x[2] =
\underbrace{1\ldots1}_{n \text{ times}}
$, for each $x \in M$.
\end{enumerate}
If any of the conditions does not hold, we set $\readout(M) =0$.
Finally, we let
$\comb(x,y)=y$.

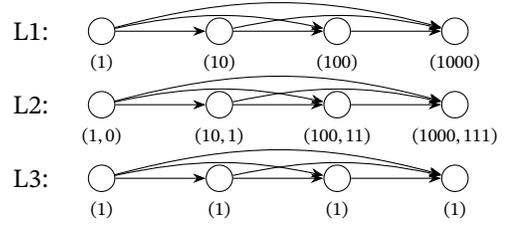
\begin{figure}[t] 
    \centering 
\begin{tikzpicture}[
  dotnode/.style={circle, draw, minimum size=10pt, inner sep=0pt},
  layerlabel/.style={anchor=east},
  >={Stealth},
  node distance=0.5cm and 1.2cm,
]

\node[layerlabel] (l1) 
at (-1.2,0) {L1:};
\node[layerlabel] (l2) [below=of l1] {L2:};
\node[layerlabel] (l3) [below=of l2] {L3:};


\node[dotnode, label=below:{\scriptsize $(1)$}]       (n10) [right=0.4cm of l1] {};
\node[dotnode, label=below:{\scriptsize$(10)$}]      (n11) [right=of n10] {};
\node[dotnode, label=below:{\scriptsize$(100)$}]     (n12) [right=of n11] {};
\node[dotnode, label=below:{\scriptsize$(1000)$}]    (n13) [right=of n12] {};

\node[dotnode, label=below:{\scriptsize $(1,0)$}]       (n20) [right=0.4cm of l2] {};
\node[dotnode, label=below:{\scriptsize $(10,1)$}]      (n21) [right=of n20] {};
\node[dotnode, label=below:{\scriptsize $(100,11)$}]    (n22) [right=of n21] {};
\node[dotnode, label=below:{\scriptsize $(1000,111)$}]  (n23) [right=of n22] {};

\node[dotnode, label=below:{\scriptsize $(1)$}] (n30) [right=0.4cm of l3] {};
\node[dotnode, label=below:{\scriptsize $(1)$}] (n31) [right=of n30] {};
\node[dotnode, label=below:{\scriptsize $(1)$}] (n32) [right=of n31] {};
\node[dotnode, label=below:{\scriptsize $(1)$}] (n33) [right=of n32] {};


\draw[->] (n10) -- (n11);
\draw[->] (n11) -- (n12);
\draw[->] (n12) -- (n13);
\draw[->, bend left=12] (n10) to (n12);
\draw[->, bend left=15] (n10) to (n13);
\draw[->, bend left=12] (n11) to (n13);

\draw[->] (n20) -- (n21);
\draw[->] (n21) -- (n22);
\draw[->] (n22) -- (n23);
\draw[->, bend left=12] (n20) to (n22);
\draw[->, bend left=15] (n20) to (n23);
\draw[->, bend left=12] (n21) to (n23);

\draw[->] (n30) -- (n31);
\draw[->] (n31) -- (n32);
\draw[->] (n32) -- (n33);
\draw[->, bend left=12] (n30) to (n32);
\draw[->, bend left=15] (n30) to (n33);
\draw[->, bend left=12] (n31) to (n33);

\end{tikzpicture}
    \caption{Application of layers 1--3 of the ACR-GNN from \Cref{LinInGNN} to the strict linear order with four nodes
    }
    \label{fig::GNN_LO}
\end{figure}

Condition (i) guarantees that each node has a different in-degree.
If this is the case, then Condition (ii)---which can be equivalently written as $\frac{x[1] - 1}{9} =  x[2]$---checks if the graph is total.
Hence, for any graph $G=(V,E,\lambda)$, if $E$ is a strict linear order, then  $\N(G,v)=1$ and otherwise  $\N(G,v)=0$, for any node $v$ in $G$.
\end{proof}

To finish this section, we need to show that $\varphi_\LO(x)$ cannot be expressed in \Ctwo{}. For this, we will exploit our bounded \WL{} algorithm and corresponding \Cref{mainWLtheorem}.

\begin{restatable}{theorem}{WLforLin}\label{thm:LonotC2}
Over directed graphs,
$\varphi_{\LO}(x)$ is not expressible in \Ctwo{}.
\end{restatable}
\begin{proof}[Proof sketch]
Suppose towards a  contradiction that $\varphi_{\LO}(x)$ is expressible in $\Ctwo$, so it is definable by a formula in $\Ctwo_{\ell, c}$, for some $\ell, c \in \mathbb{N}$.
To obtain a contradiction, we will construct a graph $G$ with nodes $v_i$ and a graph $G'$ with corresponding nodes $v_i'$, such that $G \models \varphi_{\LO}(v_i)$ and $G' \not\models \varphi_{\LO}(v_i')$, but $G,v_i \equiv_{\Ctwo_{l,c}} G',v_i$ for all nodes $v_i$.

Let $n = \ell \cdot c +1$.
We define $G=(V,E,\lambda)$ as a strict linear order over $2n+1$ nodes $V=\{v_{-n}, \dots, v_n \}$, with  $E =\{ (v_i,v_j) : i < j \}$, and $\lambda(v_i)=0$ for each $v_i$.
We let $G'=(V',E',\lambda')$ be such that 
$V'=\{v'_{-n}, \dots, v'_n \}$, $E' =\{ (v'_i,v'_j) : i < j \} \setminus \{ (v'_{-1},v'_{1}) \} \cup \{ (v'_{1},v'_{-1}) \}$, and $\lambda'(v_i')=0$ for each $v_i'$.
For example, if $c=2$ and $\ell =2$, the graphs $G$ is  depicted on top of \Cref{fig::WL_LO}; graph $G'$ is similar, but instead of $(v'_{-1},v'_{1})$ it has the opposite edge $(v'_{1},v'_{-1})$.
Notice that both graphs are irreflexive, asymmetric, and total, but only $G$ is transitive.
Hence, for all nodes $v_i$, we have $G \models \varphi_{\LO}(v_i)$ and $G' \not\models \varphi_{\LO}(v_i')$.

\begin{figure}[t] 
    \centering 
\begin{tikzpicture}[
  dotnode/.style={circle, draw, minimum size=10pt, inner sep=0pt},
  layerlabel/.style={anchor=east},
  >={Stealth},
  node distance=0.5cm and 0.3cm,
]

\node[layerlabel] (l0) at (-1.2,0) {$W^0_2$:};
\node[layerlabel] (l1) [below=of l0] {$W^1_2$:};
\node[layerlabel] (l2) [below=of l1] {$W^2_2$:};

\node[dotnode, label=below:{\scriptsize  $v_{-5}$}] (n00) [right=0.1cm of l0] {};

\foreach \i [evaluate=\i as \prev using int(\i-1),
             evaluate=\i as \labelval using int(\i - 5)] in {1,...,10} {
  \node[dotnode, label=below:{\scriptsize  $v_{\labelval}$}] (n0\i) [right=of n0\prev] {};
}

\node[dotnode, label=below:{\scriptsize  $v_{-5}$}] (n10) [right=0.1cm of l1] {};
\foreach \i [evaluate=\i as \prev using int(\i-1),
             evaluate=\i as \labelval using int(\i - 5)] in {1,...,10} {
  \node[dotnode, label=below:{\scriptsize  $v_{\labelval}$}] (n1\i) [right=of n1\prev] {};
}
\node[dotnode, fill=red!60] at (n10) {};
\node[dotnode, fill=cyan!60!blue] at (n11) {};

\node[dotnode, fill=olive!70!black] at (n19) {};
\node[dotnode, fill=violet!70] at (n110) {};

\node[dotnode, label=below:{\scriptsize  $v_{-5}$}] (n20) [right=0.1cm of l2] {};

\foreach \i [evaluate=\i as \prev using int(\i-1),
             evaluate=\i as \labelval using int(\i - 5)] in {1,...,10} {
  \node[dotnode, label=below:{\scriptsize  $v_{\labelval}$}] (n2\i) [right=of n2\prev] {};
}
\node[dotnode, fill=red!60] at (n20) {};
\node[dotnode, fill=cyan!60!blue] at (n21) {};
\node[dotnode, fill=orange!80] at (n22) {};
\node[dotnode, fill=gray!40!blue!30] at (n23) {};

\node[dotnode, fill=brown!30!yellow!60] at (n27) {};
\node[dotnode, fill=magenta!70] at (n28) {};
\node[dotnode, fill=olive!70!black] at (n29) {};
\node[dotnode, fill=violet!70] at (n210) {};

\foreach \i [evaluate=\i as \next using int(\i+1)] in {0,...,9} {
  \draw[->] (n0\i) -- (n0\next);
}
\draw[->, bend left=35] (n04) to (n06);

\foreach \i [evaluate=\i as \next using int(\i+1)] in {0,...,9} {
  \draw[->] (n1\i) -- (n1\next);
}
\draw[->, bend left=35] (n14) to (n16);

\foreach \i [evaluate=\i as \next using int(\i+1)] in {0,...,9} {
  \draw[->] (n2\i) -- (n2\next);
}
\draw[->, bend left=35] (n24) to (n26);

\end{tikzpicture}
    \caption{Application of $\WL_c$ to $G$ from \Cref{thm:LonotC2}; for  readability we draw only arrows $(v_i,v_{i+1})$  between consecutive nodes and $(v_{-1}, v_1)$ distinguishing $G$ from $G'$}
    \label{fig::WL_LO}
\end{figure}
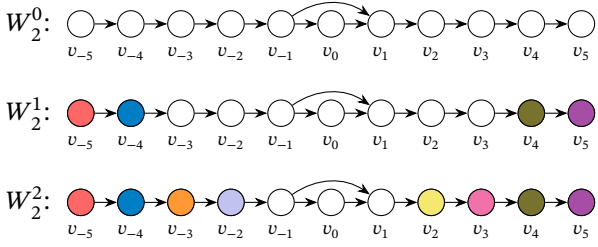

It remains to show that $G,v_i \equiv_{\Ctwo_{l,c}} G',v'_i$.
To this end, by \Cref{mainWLtheorem},
 it suffices to show that $W^\ell_c (v_i) = W^\ell_c (v_i')$.
We can prove it by showing, with a simultaneous induction on $k \leq \ell$,  the  following  two statements:
\begin{enumerate}[label=(\roman*), leftmargin=*, align=left, labelsep=-0.5em]
\item
$W^k_c (v_i) = W^k_c (v_i')$, for  $i \in \{-n, \dots, n\}$,

\item
$W^k_c (v_i) = W^k_c (v_j)$, for  $i,j \in \{ -(n - ck), \dots,  n -ck   \}$.
\end{enumerate}
Statement (ii) ensures that all  `middle nodes' have the same colour;
for instance  in  \Cref{fig::WL_LO}  nodes $v_{-3}, \dots, v_3$ have the same colour in $W_2^1$.
We use it to show Statement~(i), which implies required
$G,v_i \equiv_{\Ctwo_{l,c}} G',v'_i$.
\end{proof}

Hence, we can conclude  this sections as follows.

\begin{corollary}
Over directed graphs,
there are \FO{} node classifiers expressible by ACR-GNNs which are not expressible in $\Ctwo$.
In particular, $\varphi_\LO(x)$ is such a classifier.
\end{corollary}

\section{Logical Expressiveness Over Undirected Graphs}\label{sec:undirected}

In this section, we consider the setting of undirected graphs.
We will
solve the open problem of \citet{DBLP:conf/iclr/BarceloKM0RS20}, asking whether over undirected graphs the \FO{} node properties expressible by ACR-GNNs are exactly those definable in $\Ctwo$.
We will show that, the answer is negative.
In particular, we will show that, similarly to the case of directed graphs  in \Cref{sec:directed}, there is a property expressible by both \FO{} and ACR-GNNs, but which cannot be expressed in \Ctwo.
Our proofs  will build on some ideas from  \Cref{sec:directed}, but no access to directed edges will require  more complex argumentation.


In place of $\vp_\LO(x)$ from \Cref{sec:directed}, we will use now classifier $\vp_\GLO(x)$.
It checks if a node belongs to a 
\emph{gadgetised linear order}, which is an undirected graph $\Gad(G)$ obtained by  encoding (gadgetising) some strict  linear order $G$.
Intuitively, $\Gad(G)$ is obtained by replacing each directed edge $(u,w)$ in $G$ with a path of three  undirected edges---called  \emph{gadgetised edges}---as depicted in \Cref{gadget}.
 Next, we present a formal definition of gadgetisation. 

\begin{figure}[t]
    \centering
\begin{adjustbox}{trim=0 0pt 0 0pt, clip}
\begin{tikzpicture}[
  dotnode/.style={circle, draw, minimum size=8pt, inner sep=0pt},
  fillerleft/.style={circle, draw, minimum size=7pt, inner sep=0pt},
  layerlabel/.style={anchor=east},
  node distance=1.5cm and 2.2cm
]

\begin{scope}
    \node[dotnode, fill=brown!30!yellow!60,  label=below:{\scriptsize $v_a^1$}] (n04) at (0,0) {\scriptsize $P_1$};
    \node[dotnode, fill=brown!30!yellow!60, label=below:{\scriptsize $v^1_b$}] (n05) [right=of n04] {\scriptsize $P_1$};
    \node[dotnode, fill=brown!30!yellow!60, label=below:{\scriptsize $v^1_c$}] (n06) [right=of n05] {\scriptsize $P_1$};
    \node[dotnode, fill=brown!30!yellow!60, label=below:{\scriptsize $v^1_d$}] (n07) [right=of n06] {\scriptsize $P_1$};

    \path (n04) -- (n05) coordinate[pos=0.33] (fL4);
    \path (n04) -- (n05) coordinate[pos=0.66] (fR4);
    \node[fillerleft, fill=red!60, label=below:{\scriptsize $v^2_{(a,b)}$}] (fL4) at (fL4) {\scriptsize $P_2$};
    \node[fillerleft, fill=cyan!60!blue, label=below:{\scriptsize $v^3_{(a,b)}$}] (fR4) at (fR4) {\scriptsize $P_3$};

    \path (n05) -- (n06) coordinate[pos=0.33] (fL5);
    \path (n05) -- (n06) coordinate[pos=0.66] (fR5);
    \node[fillerleft, fill=red!60, label=below:{\scriptsize $v^2_{(b,c)}$}] (fL5) at (fL5) {\scriptsize $P_2$};
    \node[fillerleft, fill=cyan!60!blue, label=below:{\scriptsize $v^3_{(b,c)}$}] (fR5) at (fR5) {\scriptsize $P_3$};
      
    \path (n06) -- (n07) coordinate[pos=0.33] (fL6);
    \path (n06) -- (n07) coordinate[pos=0.66] (fR6);
    \node[fillerleft, fill=red!60, label=below:{\scriptsize $v^2_{(c,d)}$}] (fL6) at (fL6) {\scriptsize $P_2$};
    \node[fillerleft, fill=cyan!60!blue, label=below:{\scriptsize $v^3_{(c,d)}$}] (fR6) at (fR6) {\scriptsize $P_3$};
      
    \foreach \i [evaluate=\i as \next using int(\i+1)] in {4,...,6} {
      \draw (n0\i) -- (fL\i) -- (fR\i) -- (n0\next);
    }  

    \node[fillerleft, fill=red!60, label=above:{\scriptsize $v^2_{(a,c)}$}] (cfL2) [above=0.05cm of fR4] {\scriptsize $P_2$};
    \node[fillerleft, fill=cyan!60!blue, label=above:{\scriptsize $v^3_{(a,c)}$}] (cfR2) [above=0.05cm of fL5] {\scriptsize $P_3$};
    \draw (n04) -- (cfL2) -- (cfR2) -- (n06);

    \node[fillerleft, fill=red!60, label=above:{\scriptsize $v^2_{(b,d)}$}] (cfL1) [above=0.05cm of fR5] {\scriptsize $P_2$};
    \node[fillerleft, fill=cyan!60!blue, label=above:{\scriptsize $v^3_{(b,d)}$}] (cfR1) [above=0.05cm of fL6] {\scriptsize $P_3$};
    \draw (n05) -- (cfL1) -- (cfR1) -- (n07);

    \node[fillerleft, fill=red!60, label=above:{\scriptsize $v^2_{(a,d)}$}] (cfL2a) [above=0.25cm of fL4] {\scriptsize $P_2$};
    \node[fillerleft, fill=cyan!60!blue, label=above:{\scriptsize $v^3_{(a,d)}$}] (cfR2a) [above=0.25cm of fR6] {\scriptsize $P_3$};
    \draw (n04) -- (cfL2a) -- (cfR2a) -- (n07);
\end{scope}

\end{tikzpicture}
\end{adjustbox}

\caption{Gadgetisation of the linear order from \Cref{fig::GNN_LO} assuming its nodes are called $a$, $b$, $c$, and $d$; labels $(1,0,0)$, $(0,1,0)$, and $(0,0,1)$ are represented as $P_1$, $P_2$, and $P_3$, respectively (and also with colours)}
\label{gadget}
\end{figure}
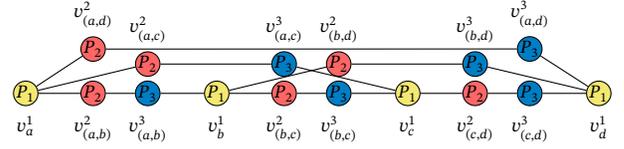

\begin{definition}
The \emph{gadgetisation}, $\Gad(G)$, of a directed graph  $G=(V,E, \lambda)$ is an undirected graph $G'=(V',E',\lambda')$ of dimension 3 such that for each edge $(u,w) \in E$, the graph $G'$ has:
\begin{itemize}
\item nodes $v_u^1$, $v_{(u,w)}^2$, $v_{(u,w)}^3$, $v_w^1$  in $V'$, 
\item edges  $\{v_u^1 ,v_{(u,w)}^2 \}$, $\{v_{(u,w)}^2,v_{(u,w)}^3 \}$, $\{ v_{(u,w)}^3, v_w^1 \}$  in $E'$,
\item  labelling of nodes with  $\lambda'(v_u^1) = \lambda'(v_w^1)  = (1,0,0)$,   $\lambda'(v_{(u,w)}^2) = (0,1,0)$, and
$\lambda'(v_{(u,w)}^3)  = (0,0,1)$.
\end{itemize}
\end{definition}
Recall that we identify undirected graphs with symmetric directed graphs, so an undirected edge, like $\{v_u^1 ,v_{(u,w)}^2 \}$ in the definition above,
can be seen as a pair of directed edges 
$(v_u^1 ,v_{(u,w)}^2 )$, $(v_{(u,w)}^2,v_u^1 )$.
Note also that our construction of $\Gad(G)$ does not depend on the labelling $\lambda$ in $G$.
Now, the formal definition of $\varphi_\GLO(x)$ is as follows:

\begin{definition}
We let $\varphi_\GLO(x)$ be a node classifier accepting a node of a  graph $G$ if and only if $G$ is isomorphic to $\Gad(G')$, for some  strict linear order $G'$.
\end{definition}
It the remaining part of this section,  we will show that $\varphi_\GLO(x)$ is expressible in \FO{} and by ACR-GNNs, but it is not expressible in \Ctwo.

\begin{restatable}{theorem}{FOgadget}\label{gadlinisfo}
Over undirected graphs, $\vp_\GLO(x)$ is expressible in $FO$.
\end{restatable}
\begin{proof}[Proof sketch] 
We will express $\vp_\GLO(x)$ as a conjunction of four $\FO$ formulas $\vp_1$, $\vp_2$, $\vp_3$, and $\vp_4$.
Recall that we identify graphs with $\FO$ structures interpreting unary predicates $P_1, \dots , P_d$, where $d$ is the dimension of the graph, and one binary predicate $E$.
Since gadgetisations are always of dimension $d=3$, our formulas will mention three unary predicated $P_1$, $P_2$, and $P_3$.

Formula $\vp_1$ states that 
$P_1$, $P_2$, and $P_3$
partition the set of nodes.
Formula $\vp_2$ states that every node satisfying $P_2$ has exactly two $E$-neighbours: one satisfying $P_1$ and the other satisfying $P_3$.
It states also that every node satisfying $P_3$ has exactly two $E$-neighbours: one satisfying $P_1$ and the other satisfying $P_2$.
Finally, it states that 
if $u$ and $v$ are nodes satisfying $P_1$, then $E(u,v)$ cannot be true.
Formulas $\vp_3$ and $\vp_4$ are about \emph{gadgetised edges}, which are paths in $\Gad(G)$ that correspond to directed edges in $G$.
In particular, we let a gadgetised edge from $u$ to $z$ be a path of the form $E(u,w)$, $E(w,v)$, $E(v,z)$ with
$P_1(u)$, $P_2(w)$, $P_3(v)$, and $P_1(z)$.
Formula $\vp_3$ states that between any two distinct nodes satisfying $P_1$ there is exactly one gadgetised edge.
Formula $\vp_4$, in turn, states that there are no nodes $u,w,v$ with gadgetised edges from $v$ to $w$, from $w$ to $u$, and from $u$ to $v$.

All formulas $\vp_1$--$\vp_4$ can be written in $\FO$, and we can show that a graph satisfies all of them if and only if the graph is a gadgetised linear order.
\end{proof}

In \Cref{gadlinisfo} we have showed how to  express $\vp_{\GLO}(x)$ with 
\FO{} formulas $\vp_1$--$\vp_4$.
We observe that $\vp_1$ and $\vp_2$ are in $\Ctwo$, so by the result of 
\citet{DBLP:conf/iclr/BarceloKM0RS20}, we can express them with ACR-GNNs. 
However  $\vp_3$ and $\vp_4$ cannot be expressed by ACR-GNNs.
However, as will show, $\vp_3$ and $\vp_4$ can be replaced with a property that  is expressible by  ACR-GNNs.
This will show that $\vp_{\GLO}(x)$ is expressible by ACR-GNNs.

\begin{restatable}{theorem}{GLOisGNN}\label{gloisacrgnn}
Over undirected graphs, 
$\vp_{\GLO}(x)$ is expressible by an ACR-GNN. 
\end{restatable}
\begin{proof}[Proof sketch]
We can show that a graph $G$ is a gadgetised linear  order if and only if $G$ satisfies $\vp_1$, $\vp_2$ (see the proof of \Cref{gadlinisfo}) and a property $\psi$ explained next.
Property $\psi$ states that for all $i<j<|P_1|$, the graph has  nodes $v_i$ and $v_j$
such that 
(1) both $v_i$ and $v_j$  satisfy $P_1$,
(2) $v_i$ has $i$ neighbours satisfying $P_2$ and $v_j$ has $j$ such neighbours, and
(3) there is a gadgetised edge (see the proof of \Cref{gadlinisfo}) from $v_j$ to $v_i$.
Since $\vp_1$ and $\vp_2$ are $\Ctwo$ formulas, they can be expressed by ACR-GNNs \cite[Theorem 5.1]{DBLP:conf/iclr/BarceloKM0RS20}.
It remains to construct an ACR-GNN $\N$ which expresses $\psi$, since
it is straightforward to combine the three ACR-GNNs into a single GNN. 

Recall that  gadgetised linear orders are graphs of dimension three, so we will consider application of $\N$ to such graphs $G$.
In each layer,  $\N$ will assign to nodes vectors of dimesion five, where the first three positions are always as in the input graph $G$, so  information about  $P_1$, $P_2$, and $P_3$ in the input graph is preserved across all layers. 
The fourth and fifth positions will always keep binary numbers.
The details of $\N$ are provided next and and example of its application is visualised in \Cref{GNNforGLO}.

\begin{figure}[ht]
    \centering
\begin{adjustbox}{trim=0 3pt 0 0pt, clip}
\begin{tikzpicture}[
  dotnode/.style={circle, draw, minimum size=8pt, inner sep=0pt},
  fillerleft/.style={circle, draw, minimum size=7pt, inner sep=0pt},
  layerlabel/.style={anchor=east},
  node distance=1.5cm and 2.2cm
]

\begin{scope}[yshift=-2cm]
    \node[dotnode, fill=brown!30!yellow!60,  label=below:{\scriptsize $(\mathbf{1000},0)$}] (n04) at (0,0) {\scriptsize $P_1$};
    \node[layerlabel] at ([xshift=0.2cm, yshift=0.7cm]n04) {L1:};
    \node[dotnode, fill=brown!30!yellow!60, label=below:{\scriptsize $(\mathbf{100},0)$}] (n05) [right=of n04] {\scriptsize $P_1$};
    \node[dotnode, fill=brown!30!yellow!60, label=below:{\scriptsize $(\mathbf{10},0)$}] (n06) [right=of n05] {\scriptsize $P_1$};
    \node[dotnode, fill=brown!30!yellow!60, label=below:{\scriptsize $(\mathbf{1},0)$}] (n07) [right=of n06] {\scriptsize $P_1$};

    \path (n04) -- (n05) coordinate[pos=0.33] (fL4);
    \path (n04) -- (n05) coordinate[pos=0.66] (fR4);
    \node[fillerleft, fill=red!60, label=below:{\scriptsize $(0,0)$}] (fL4) at (fL4) {\scriptsize $P_2$};
    \node[fillerleft, fill=cyan!60!blue, label=below:{\scriptsize $(0,0)$}] (fR4) at (fR4) {\scriptsize $P_3$};

    \path (n05) -- (n06) coordinate[pos=0.33] (fL5);
    \path (n05) -- (n06) coordinate[pos=0.66] (fR5);
    \node[fillerleft, fill=red!60, label=below:{\scriptsize $(0,0)$}] (fL5) at (fL5) {\scriptsize $P_2$};
    \node[fillerleft, fill=cyan!60!blue, label=below:{\scriptsize $(0,0)$}] (fR5) at (fR5) {\scriptsize $P_3$};
      
    \path (n06) -- (n07) coordinate[pos=0.33] (fL6);
    \path (n06) -- (n07) coordinate[pos=0.66] (fR6);
    \node[fillerleft, fill=red!60, label=below:{\scriptsize $(0,0)$}] (fL6) at (fL6) {\scriptsize $P_2$};
    \node[fillerleft, fill=cyan!60!blue, label=below:{\scriptsize $(0,0)$}] (fR6) at (fR6) {\scriptsize $P_3$};
      
    \foreach \i [evaluate=\i as \next using int(\i+1)] in {4,...,6} {
      \draw (n0\i) -- (fL\i) -- (fR\i) -- (n0\next);
    }  

    \node[fillerleft, fill=red!60, label=above:{\scriptsize $(0,0)$}] (cfL2) [above=0.05cm of fR4] {\scriptsize $P_2$};
    \node[fillerleft, fill=cyan!60!blue, label=above:{\scriptsize $(0,0)$}] (cfR2) [above=0.05cm of fL5] {\scriptsize $P_3$};
    \draw (n04) -- (cfL2) -- (cfR2) -- (n06);

    \node[fillerleft, fill=red!60, label=above:{\scriptsize $(0,0)$}] (cfL1) [above=0.05cm of fR5] {\scriptsize $P_2$};
    \node[fillerleft, fill=cyan!60!blue, label=above:{\scriptsize $(0,0)$}] (cfR1) [above=0.05cm of fL6] {\scriptsize $P_3$};
    \draw (n05) -- (cfL1) -- (cfR1) -- (n07);

    \node[fillerleft, fill=red!60, label=above:{\scriptsize $(0,0)$}] (cfL2a) [above=0.25cm of fL4] {\scriptsize $P_2$};
    \node[fillerleft, fill=cyan!60!blue, label=above:{\scriptsize $(0,0)$}] (cfR2a) [above=0.25cm of fR6] {\scriptsize $P_3$};
    \draw (n04) -- (cfL2a) -- (cfR2a) -- (n07);
\end{scope}

\begin{scope}[yshift=-4cm]
    \node[dotnode, fill=brown!30!yellow!60,  label=below:{\scriptsize $(1000,0)$}] (n04) at (0,0) {\scriptsize $P_1$};
    \node[layerlabel] at ([xshift=0.2cm, yshift=0.7cm]n04) {L2:};
    \node[dotnode, fill=brown!30!yellow!60, label=below:{\scriptsize $(100,0)$}] (n05) [right=of n04] {\scriptsize $P_1$};
    \node[dotnode, fill=brown!30!yellow!60, label=below:{\scriptsize $(10,0)$}] (n06) [right=of n05] {\scriptsize $P_1$};
    \node[dotnode, fill=brown!30!yellow!60, label=below:{\scriptsize $(1,0)$}] (n07) [right=of n06] {\scriptsize $P_1$};

    \path (n04) -- (n05) coordinate[pos=0.33] (fL4);
    \path (n04) -- (n05) coordinate[pos=0.66] (fR4);
    \node[fillerleft, fill=red!60, label=below:{\scriptsize $(0,0)$}] (fL4) at (fL4) {\scriptsize $P_2$};
    \node[fillerleft, fill=cyan!60!blue, label=below:{\scriptsize $(\mathbf{100},0)$}] (fR4) at (fR4) {\scriptsize $P_3$};

    \path (n05) -- (n06) coordinate[pos=0.33] (fL5);
    \path (n05) -- (n06) coordinate[pos=0.66] (fR5);
    \node[fillerleft, fill=red!60, label=below:{\scriptsize $(0,0)$}] (fL5) at (fL5) {\scriptsize $P_2$};
    \node[fillerleft, fill=cyan!60!blue, label=below:{\scriptsize $(\mathbf{10},0)$}] (fR5) at (fR5) {\scriptsize $P_3$};
      
    \path (n06) -- (n07) coordinate[pos=0.33] (fL6);
    \path (n06) -- (n07) coordinate[pos=0.66] (fR6);
    \node[fillerleft, fill=red!60, label=below:{\scriptsize $(0,0)$}] (fL6) at (fL6) {\scriptsize $P_2$};
    \node[fillerleft, fill=cyan!60!blue, label=below:{\scriptsize $(\mathbf{1},0)$}] (fR6) at (fR6) {\scriptsize $P_3$};
      
    \foreach \i [evaluate=\i as \next using int(\i+1)] in {4,...,6} {
      \draw (n0\i) -- (fL\i) -- (fR\i) -- (n0\next);
    }  

    \node[fillerleft, fill=red!60, label=above:{\scriptsize $(0,0)$}] (cfL2) [above=0.05cm of fR4] {\scriptsize $P_2$};
    \node[fillerleft, fill=cyan!60!blue, label=above:{\scriptsize $(\mathbf{10},0)$}] (cfR2) [above=0.05cm of fL5] {\scriptsize $P_3$};
    \draw (n04) -- (cfL2) -- (cfR2) -- (n06);

    \node[fillerleft, fill=red!60, label=above:{\scriptsize $(0,0)$}] (cfL1) [above=0.05cm of fR5] {\scriptsize $P_2$};
    \node[fillerleft, fill=cyan!60!blue, label=above:{\scriptsize $(\mathbf{1},0)$}] (cfR1) [above=0.05cm of fL6] {\scriptsize $P_3$};
    \draw (n05) -- (cfL1) -- (cfR1) -- (n07);

    \node[fillerleft, fill=red!60, label=above:{\scriptsize $(0,0)$}] (cfL2a) [above=0.25cm of fL4] {\scriptsize $P_2$};
    \node[fillerleft, fill=cyan!60!blue, label=above:{\scriptsize $(\mathbf{1},0)$}] (cfR2a) [above=0.25cm of fR6] {\scriptsize $P_3$};
    \draw (n04) -- (cfL2a) -- (cfR2a) -- (n07);
\end{scope}

\begin{scope}[yshift=-6cm]
    \node[dotnode, fill=brown!30!yellow!60,  label=below:{\scriptsize \hspace{-3mm} $(1000,0)$}] (n04) at (0,0) {\scriptsize $P_1$};
    \node[layerlabel] at ([xshift=0.2cm, yshift=0.7cm]n04) {L3:};
    \node[dotnode, fill=brown!30!yellow!60, label=below:{\scriptsize $(100,0)$}] (n05) [right=of n04] {\scriptsize $P_1$};
    \node[dotnode, fill=brown!30!yellow!60, label=below:{\scriptsize $(10,0)$}] (n06) [right=of n05] {\scriptsize $P_1$};
    \node[dotnode, fill=brown!30!yellow!60, label=below:{\scriptsize $(1,0)$}] (n07) [right=of n06] {\scriptsize $P_1$};

    \path (n04) -- (n05) coordinate[pos=0.33] (fL4);
    \path (n04) -- (n05) coordinate[pos=0.66] (fR4);
    \node[fillerleft, fill=red!60, label=below:{\scriptsize\hspace{-2mm}  $(\mathbf{100},0)$}] (fL4) at (fL4) {\scriptsize $P_2$};
    \node[fillerleft, fill=cyan!60!blue, label=below:{\scriptsize $(100,0)$}] (fR4) at (fR4) {\scriptsize $P_3$};

    \path (n05) -- (n06) coordinate[pos=0.33] (fL5);
    \path (n05) -- (n06) coordinate[pos=0.66] (fR5);
    \node[fillerleft, fill=red!60, label=below:{\scriptsize $(\mathbf{10},0)$}] (fL5) at (fL5) {\scriptsize $P_2$};
    \node[fillerleft, fill=cyan!60!blue, label=below:{\scriptsize $(10,0)$}] (fR5) at (fR5) {\scriptsize $P_3$};
      
    \path (n06) -- (n07) coordinate[pos=0.33] (fL6);
    \path (n06) -- (n07) coordinate[pos=0.66] (fR6);
    \node[fillerleft, fill=red!60, label=below:{\scriptsize $(\mathbf{1},0)$}] (fL6) at (fL6) {\scriptsize $P_2$};
    \node[fillerleft, fill=cyan!60!blue, label=below:{\scriptsize $(1,0)$}] (fR6) at (fR6) {\scriptsize $P_3$};
      
    \foreach \i [evaluate=\i as \next using int(\i+1)] in {4,...,6} {
      \draw (n0\i) -- (fL\i) -- (fR\i) -- (n0\next);
    }  

    \node[fillerleft, fill=red!60, label=above:{\scriptsize $(\mathbf{10},0)$}] (cfL2) [above=0.05cm of fR4] {\scriptsize $P_2$};
    \node[fillerleft, fill=cyan!60!blue, label=above:{\scriptsize $(10,0)$}] (cfR2) [above=0.05cm of fL5] {\scriptsize $P_3$};
    \draw (n04) -- (cfL2) -- (cfR2) -- (n06);

    \node[fillerleft, fill=red!60, label=above:{\scriptsize $(\mathbf{1},0)$}] (cfL1) [above=0.05cm of fR5] {\scriptsize $P_2$};
    \node[fillerleft, fill=cyan!60!blue, label=above:{\scriptsize $(1,0)$}] (cfR1) [above=0.05cm of fL6] {\scriptsize $P_3$};
    \draw (n05) -- (cfL1) -- (cfR1) -- (n07);

    \node[fillerleft, fill=red!60, label=above:{\scriptsize $(\mathbf{1},0)$}] (cfL2a) [above=0.25cm of fL4] {\scriptsize $P_2$};
    \node[fillerleft, fill=cyan!60!blue, label=above:{\scriptsize $(1,0)$}] (cfR2a) [above=0.25cm of fR6] {\scriptsize $P_3$};
    \draw (n04) -- (cfL2a) -- (cfR2a) -- (n07);
\end{scope}

\begin{scope}[yshift=-8cm]
    \node[dotnode, fill=brown!30!yellow!60,  label=below:{\scriptsize\hspace{-1mm} $(1000,\textbf{111})$}] (n04) at (0,0) {\scriptsize $P_1$};
    \node[layerlabel] at ([xshift=0.2cm, yshift=0.7cm]n04) {L4:};
    \node[dotnode, fill=brown!30!yellow!60, label=below:{\scriptsize \hspace{1mm} $(100,\textbf{11})$}] (n05) [right=of n04] {\scriptsize $P_1$};
    \node[dotnode, fill=brown!30!yellow!60, label=below:{\scriptsize $(10,\textbf{1})$}] (n06) [right=of n05] {\scriptsize $P_1$};
    \node[dotnode, fill=brown!30!yellow!60, label=below:{\scriptsize $(1,\textbf{0})$}] (n07) [right=of n06] {\scriptsize $P_1$};

    \path (n04) -- (n05) coordinate[pos=0.33] (fL4);
    \path (n04) -- (n05) coordinate[pos=0.66] (fR4);
    \node[fillerleft, fill=red!60, label=below:{\scriptsize \hspace{1mm} $(100,0)$}] (fL4) at (fL4) {\scriptsize $P_2$};
    \node[fillerleft, fill=cyan!60!blue, label=below:{\scriptsize\hspace{2mm} $(100,0)$}] (fR4) at (fR4) {\scriptsize $P_3$};

    \path (n05) -- (n06) coordinate[pos=0.33] (fL5);
    \path (n05) -- (n06) coordinate[pos=0.66] (fR5);
    \node[fillerleft, fill=red!60, label=below:{\scriptsize $(10,0)$}] (fL5) at (fL5) {\scriptsize $P_2$};
    \node[fillerleft, fill=cyan!60!blue, label=below:{\scriptsize $(10,0)$}] (fR5) at (fR5) {\scriptsize $P_3$};
      
    \path (n06) -- (n07) coordinate[pos=0.33] (fL6);
    \path (n06) -- (n07) coordinate[pos=0.66] (fR6);
    \node[fillerleft, fill=red!60, label=below:{\scriptsize $(1,0)$}] (fL6) at (fL6) {\scriptsize $P_2$};
    \node[fillerleft, fill=cyan!60!blue, label=below:{\scriptsize $(1,0)$}] (fR6) at (fR6) {\scriptsize $P_3$};
      
    \foreach \i [evaluate=\i as \next using int(\i+1)] in {4,...,6} {
      \draw (n0\i) -- (fL\i) -- (fR\i) -- (n0\next);
    }  

    \node[fillerleft, fill=red!60, label=above:{\scriptsize $(10,0)$}] (cfL2) [above=0.05cm of fR4] {\scriptsize $P_2$};
    \node[fillerleft, fill=cyan!60!blue, label=above:{\scriptsize $(10,0)$}] (cfR2) [above=0.05cm of fL5] {\scriptsize $P_3$};
    \draw (n04) -- (cfL2) -- (cfR2) -- (n06);

    \node[fillerleft, fill=red!60, label=above:{\scriptsize $(1,0)$}] (cfL1) [above=0.05cm of fR5] {\scriptsize $P_2$};
    \node[fillerleft, fill=cyan!60!blue, label=above:{\scriptsize $(1,0)$}] (cfR1) [above=0.05cm of fL6] {\scriptsize $P_3$};
    \draw (n05) -- (cfL1) -- (cfR1) -- (n07);

    \node[fillerleft, fill=red!60, label=above:{\scriptsize $(1,0)$}] (cfL2a) [above=0.25cm of fL4] {\scriptsize $P_2$};
    \node[fillerleft, fill=cyan!60!blue, label=above:{\scriptsize $(1,0)$}] (cfR2a) [above=0.25cm of fR6] {\scriptsize $P_3$};
    \draw (n04) -- (cfL2a) -- (cfR2a) -- (n07);
\end{scope}

\begin{scope}[yshift=-10cm]
    \node[dotnode, fill=brown!30!yellow!60,  label=below:{\scriptsize $(\textbf{1})$}] (n04) at (0,0) {\scriptsize $P_1$};
    \node[layerlabel] at ([xshift=0.2cm, yshift=0.7cm]n04) {L5:};
    \node[dotnode, fill=brown!30!yellow!60, label=below:{\scriptsize $(\textbf{1})$}] (n05) [right=of n04] {\scriptsize $P_1$};
    \node[dotnode, fill=brown!30!yellow!60, label=below:{\scriptsize $(\textbf{1})$}] (n06) [right=of n05] {\scriptsize $P_1$};
    \node[dotnode, fill=brown!30!yellow!60, label=below:{\scriptsize $(\textbf{1})$}] (n07) [right=of n06] {\scriptsize $P_1$};

    \path (n04) -- (n05) coordinate[pos=0.33] (fL4);
    \path (n04) -- (n05) coordinate[pos=0.66] (fR4);
    \node[fillerleft, fill=red!60, label=below:{\scriptsize $(\textbf{1})$}] (fL4) at (fL4) {\scriptsize $P_2$};
    \node[fillerleft, fill=cyan!60!blue, label=below:{\scriptsize $(\textbf{1})$}] (fR4) at (fR4) {\scriptsize $P_3$};

    \path (n05) -- (n06) coordinate[pos=0.33] (fL5);
    \path (n05) -- (n06) coordinate[pos=0.66] (fR5);
    \node[fillerleft, fill=red!60, label=below:{\scriptsize $(\textbf{1})$}] (fL5) at (fL5) {\scriptsize $P_2$};
    \node[fillerleft, fill=cyan!60!blue, label=below:{\scriptsize $(\textbf{1})$}] (fR5) at (fR5) {\scriptsize $P_3$};
      
    \path (n06) -- (n07) coordinate[pos=0.33] (fL6);
    \path (n06) -- (n07) coordinate[pos=0.66] (fR6);
    \node[fillerleft, fill=red!60, label=below:{\scriptsize $(\textbf{1})$}] (fL6) at (fL6) {\scriptsize $P_2$};
    \node[fillerleft, fill=cyan!60!blue, label=below:{\scriptsize $(\textbf{1})$}] (fR6) at (fR6) {\scriptsize $P_3$};
      
    \foreach \i [evaluate=\i as \next using int(\i+1)] in {4,...,6} {
      \draw (n0\i) -- (fL\i) -- (fR\i) -- (n0\next);
    }  

    \node[fillerleft, fill=red!60, label=above:{\scriptsize $(\textbf{1})$}] (cfL2) [above=0.05cm of fR4] {\scriptsize $P_2$};
    \node[fillerleft, fill=cyan!60!blue, label=above:{\scriptsize $(\textbf{1})$}] (cfR2) [above=0.05cm of fL5] {\scriptsize $P_3$};
    \draw (n04) -- (cfL2) -- (cfR2) -- (n06);

    \node[fillerleft, fill=red!60, label=above:{\scriptsize $(\textbf{1})$}] (cfL1) [above=0.05cm of fR5] {\scriptsize $P_2$};
    \node[fillerleft, fill=cyan!60!blue, label=above:{\scriptsize $(\textbf{1})$}] (cfR1) [above=0.05cm of fL6] {\scriptsize $P_3$};
    \draw (n05) -- (cfL1) -- (cfR1) -- (n07);

    \node[fillerleft, fill=red!60, label=above:{\scriptsize $(\textbf{1})$}] (cfL2a) [above=0.25cm of fL4] {\scriptsize $P_2$};
    \node[fillerleft, fill=cyan!60!blue, label=above:{\scriptsize $(\textbf{1})$}] (cfR2a) [above=0.25cm of fR6] {\scriptsize $P_3$};
    \draw (n04) -- (cfL2a) -- (cfR2a) -- (n07);
\end{scope}

\end{tikzpicture}
\end{adjustbox}
\caption{Application of the ACR-GNN from Theorem \ref{gloisacrgnn} to
the graph from \Cref{gadget}; we present only the fourth and fifth components of vectors, and write in bold values updated in a given layer}
\label{GNNforGLO}
\end{figure}
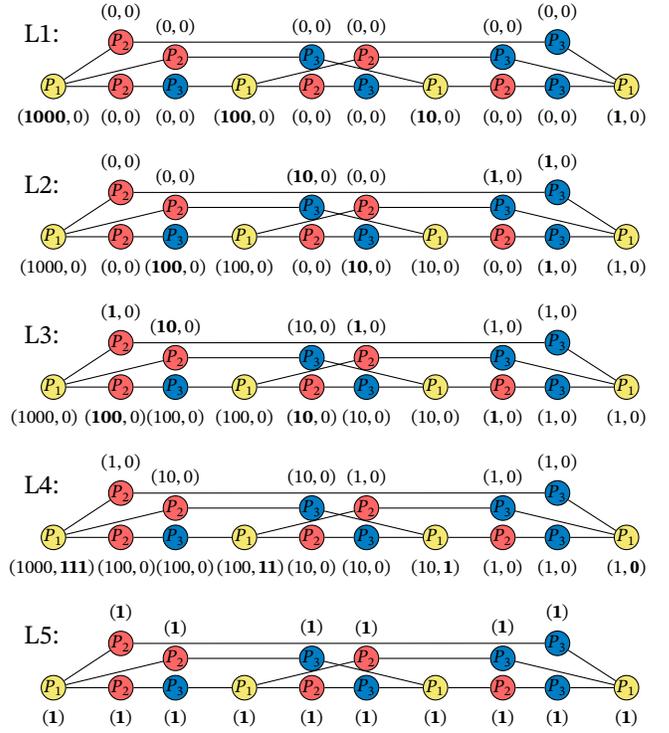

The first layer assigns to the fourth position of nodes $v$ satisfying $P_1$ the number $10^n$,  where $n$ is the number of 
neighbours of $v$ satisfying $P_2$.
Fourth and fifth positions of other nodes are set to $0$.
The next three layers will compute bitwise $OR$ applied to binary numbers,  for example  $OR(100, 10, 10) = 110$.
The second layer
assigns to the fourth position of nodes $v$ satisfying $P_3$ the value of $OR$ over the fourth positions of $v$ neighbours satisfying satisfy $P_1$.
The third layer
assigns to the fourth position of nodes $v$ satisfying $P_2$ the value of $OR$ over the fourth positions of $v$ neighbours satisfying $P_3$.
The fourth layer
assigns to the fifth position of nodes $v$ satisfying $P_1$ the value of $OR$ over the fourth positions of $v$ neighbours satisfying $P_2$.
Finally, the fifth 
layer uses a global readout to assign 1 to each node if 
for all $i<j<|P_1|$ there exists a node whose fourth position of the vector is $10^j$ and the fifth position of the vector has $1$ as the $i$th bit from the right (when counting from 0).



The  first four layers can be implemented without readout functions.
The fifth layer, in contrast, requires using readout, but no aggregation.
To show that the construction is correct, we
can show that in layer 4, each node $v$ satisfying $P_1$ has on the fourth position of its vector  $10^j$, where $j$ is the number of $v$ neighbours satisfying $P_2$.
On the fifth position $v$ has a binary number, whose $i$th bit is $1$ if there is a gadgetised edge from $v$ to some node with $i$ neighbours satisfying $P_2$.
Therefore, the fifth layer assigns 1 to all nodes if the graphs satisfies $\psi$, and otherwise it assigns $0$ to all nodes.
\end{proof}

To finish this section, it remains to show that gadgetised linear orders are not expressible in \Ctwo{}.
To this end, we will again use bounded WL from \Cref{sec:WL}, as it is applicable to both directed and undirected graphs.

\begin{restatable}{theorem}{GLonotC}\label{thm:GLonotC2}
Over undirected graphs,
the classifier $\varphi_{\GLO{}}(x)$ is not expressible in $\Ctwo$.
\end{restatable}
\begin{proof}[Proof sketch]
The proof is  similar to the one for \Cref{thm:LonotC2},
namely we suppose towards a contradiction that 
$\varphi_{\GLO}(x)$ is expressible  by a $C^2_{\ell, c}$ formula, for some $\ell,c \in \mathbb{N}$.
In the proof of \Cref{thm:LonotC2} we have obtained contradiction by 
applying $WL_c$ to  directed graphs $G$ and $G'$.
Now, we will apply $WL_c$ to their  gadgetisations $H=\Gad(G)$ and $H'=\Gad(G')$.
Since $G$ is a strict linear order, but $G'$ is not, we obtain that $H$ is a gadgetised linear order, but $H'$ is not.
Hence,  by \Cref{mainWLtheorem},
it remains to show that 
$W^\ell_c$ outputs the same colourings on $H$ and $H'$.
The proof is similar as in \Cref{thm:LonotC2}.  
Colourings obtained by applying  $W^\ell_c$ to $H$ are presented in \Cref{fig::WL_LOGLO};
application of $W^\ell_c$ to $H'$ results in the exactly same colourings.
\end{proof}





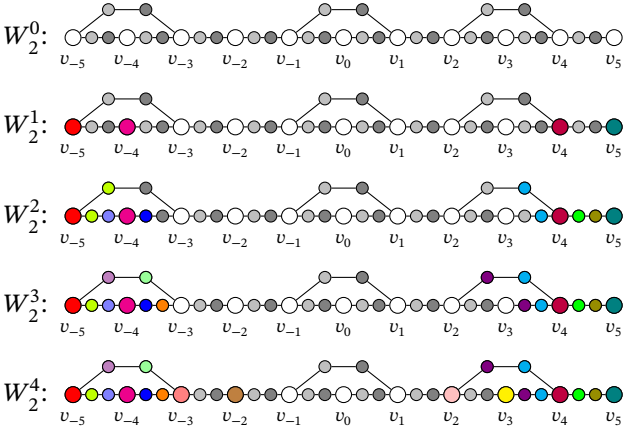
\begin{figure}[t] 
    \centering 
\begin{tikzpicture}[
  dotnode/.style={circle, draw, minimum size=6pt, inner sep=0pt,fill=white},
  smalldotnode/.style={circle, draw, minimum size=4.5pt, inner sep=0pt, fill=black!25},
  squaresmalldotnode/.style={circle, draw, minimum size=4.5pt, inner sep=0pt, fill=black!50},
  layerlabel/.style={anchor=east},
  >={Stealth},
  node distance=0.5cm and 0.05cm,
]

\node[layerlabel] (l0) at (-1.8,0) {$W^0_2$:};
\node[layerlabel] (l1) [below=of l0] {$W^1_2$:};
\node[layerlabel] (l2) [below=of l1] {$W^2_2$:};
\node[layerlabel] (l3) [below=of l2] {$W^3_2$:};
\node[layerlabel] (l4) [below=of l3] {$W^4_2$:};

\node[dotnode, label=below:{\scriptsize  $v_{-5}$}] (n00) [right=0.1cm of l0] {};
\foreach \i [evaluate=\i as \prev using int(\i-1),
             evaluate=\i as \labelval using int((\i - 15)/3),
             evaluate=\i as \smallcount using int((\i - 1) - floor((\i - 1)/3))] in {1,...,30} {
  \ifnum\i=1
    \node[smalldotnode] (n0\i) [right=of n00] {};
  \else
    \pgfmathparse{mod(\i,3)==0}
    \ifnum\pgfmathresult=1
      \node[dotnode, label=below:{\scriptsize $v_{\labelval}$}] (n0\i) [right=of n0\prev] {};
    \else
      \pgfmathparse{mod(\smallcount,2)==0}
      \ifnum\pgfmathresult=1
        \node[smalldotnode] (n0\i) [right=of n0\prev] {};
      \else
        \node[squaresmalldotnode] (n0\i) [right=of n0\prev] {};
      \fi
    \fi
  \fi
}
\node[smalldotnode] (m1) [above=0.2cm of n014] {};
\node[squaresmalldotnode] (m2) [above=0.2cm of n016] {};
\draw (m1) -- (m2);
\draw (m1) -- (n012);
\draw (n018) -- (m2);

\node[smalldotnode] (p11) [above=0.2cm of n02] {};
\node[squaresmalldotnode] (p12) [above=0.2cm of n04] {};
\draw (p11) -- (p12);
\draw (p11) -- (n00);
\draw (n06) -- (p12);

\node[smalldotnode] (q11) [above=0.2cm of n023] {};
\node[squaresmalldotnode] (q12) [above=0.2cm of n025] {};
\draw (q11) -- (q12);
\draw (q11) -- (n021);
\draw (n027) -- (q12);

\foreach \i [evaluate=\i as \prev using int(\i-1)] in {1,...,30} {
  \draw (n0\prev) -- (n0\i);
}

\node[dotnode, label=below:{\scriptsize  $v_{-5}$}] (n10) [right=0.1cm of l1] {};
\foreach \i [evaluate=\i as \prev using int(\i-1),
             evaluate=\i as \labelval using int((\i - 15)/3),
             evaluate=\i as \smallcount using int((\i - 1) - floor((\i - 1)/3))] in {1,...,30} {
  \ifnum\i=1
    \node[smalldotnode] (n1\i) [right=of n10] {};
  \else
    \pgfmathparse{mod(\i,3)==0}
    \ifnum\pgfmathresult=1
      \node[dotnode, label=below:{\scriptsize $v_{\labelval}$}] (n1\i) [right=of n1\prev] {};
    \else
      \pgfmathparse{mod(\smallcount,2)==0}
      \ifnum\pgfmathresult=1
        \node[smalldotnode] (n1\i) [right=of n1\prev] {};
      \else
        \node[squaresmalldotnode] (n1\i) [right=of n1\prev] {};
      \fi
    \fi
  \fi
}
\node[smalldotnode] (m11) [above=0.2cm of n114] {};
\node[squaresmalldotnode] (m12) [above=0.2cm of n116] {};
\draw (m11) -- (m12);
\draw (m11) -- (n112);
\draw (n118) -- (m12);

\node[smalldotnode] (p11) [above=0.2cm of n12] {};
\node[squaresmalldotnode] (p12) [above=0.2cm of n14] {};
\draw (p11) -- (p12);
\draw (p11) -- (n10);
\draw (n16) -- (p12);

\node[smalldotnode] (q11) [above=0.2cm of n123] {};
\node[squaresmalldotnode] (q12) [above=0.2cm of n125] {};
\draw (q11) -- (q12);
\draw (q11) -- (n121);
\draw (n127) -- (q12);

\foreach \i [evaluate=\i as \prev using int(\i-1)] in {1,...,30} {
  \draw (n1\prev) -- (n1\i);
}

\node[dotnode, fill=red] at (n10) {};
\node[dotnode, fill=magenta] at (n13) {};

\node[dotnode, fill=purple] at (n127) {};
\node[dotnode, fill=teal] at (n130) {};

\node[dotnode, label=below:{\scriptsize  $v_{-5}$}] (n20) [right=0.1cm of l2] {};
\foreach \i [evaluate=\i as \prev using int(\i-1),
             evaluate=\i as \labelval using int((\i - 15)/3),
             evaluate=\i as \smallcount using int((\i - 1) - floor((\i - 1)/3))] in {1,...,30} {
  \ifnum\i=1
    \node[smalldotnode] (n2\i) [right=of n20] {};
  \else
    \pgfmathparse{mod(\i,3)==0}
    \ifnum\pgfmathresult=1
      \node[dotnode, label=below:{\scriptsize $v_{\labelval}$}] (n2\i) [right=of n2\prev] {};
    \else
      \pgfmathparse{mod(\smallcount,2)==0}
      \ifnum\pgfmathresult=1
        \node[smalldotnode] (n2\i) [right=of n2\prev] {};
      \else
        \node[squaresmalldotnode] (n2\i) [right=of n2\prev] {};
      \fi
    \fi
  \fi
}
\node[smalldotnode] (m21) [above=0.2cm of n214] {};
\node[squaresmalldotnode] (m22) [above=0.2cm of n216] {};
\draw (m21) -- (m22);
\draw (m21) -- (n212);
\draw (n218) -- (m22);

\node[smalldotnode] (p11) [above=0.2cm of n22] {};
\node[squaresmalldotnode] (p12) [above=0.2cm of n24] {};
\draw (p11) -- (p12);
\draw (p11) -- (n20);
\draw (n26) -- (p12);

\node[smalldotnode] (q11) [above=0.2cm of n223] {};
\node[squaresmalldotnode] (q12) [above=0.2cm of n225] {};
\draw (q11) -- (q12);
\draw (q11) -- (n221);
\draw (n227) -- (q12);

\node[squaresmalldotnode, fill=cyan] at (q12) {};
\node[squaresmalldotnode, fill=lime] at (p11) {};

\foreach \i [evaluate=\i as \prev using int(\i-1)] in {1,...,30} {
  \draw (n2\prev) -- (n2\i);
}

\node[dotnode, fill=red] at (n20) {};
\node[smalldotnode, fill=lime] at (n21) {};
\node[squaresmalldotnode, fill=blue!50] at (n22) {};
\node[dotnode, fill=magenta] at (n23) {};
\node[smalldotnode, fill=blue] at (n24) {};

\node[squaresmalldotnode, fill=cyan] at (n226) {};
\node[dotnode, fill=purple] at (n227) {};
\node[smalldotnode, fill=green] at (n228) {};
\node[squaresmalldotnode, fill=olive] at (n229) {};
\node[dotnode, fill=teal] at (n230) {};

\node[dotnode, label=below:{\scriptsize  $v_{-5}$}] (n30) [right=0.1cm of l3] {};
\foreach \i [evaluate=\i as \prev using int(\i-1),
             evaluate=\i as \labelval using int((\i - 15)/3),
             evaluate=\i as \smallcount using int((\i - 1) - floor((\i - 1)/3))] in {1,...,30} {
  \ifnum\i=1
    \node[smalldotnode] (n3\i) [right=of n30] {};
  \else
    \pgfmathparse{mod(\i,3)==0}
    \ifnum\pgfmathresult=1
      \node[dotnode, label=below:{\scriptsize $v_{\labelval}$}] (n3\i) [right=of n3\prev] {};
    \else
      \pgfmathparse{mod(\smallcount,2)==0}
      \ifnum\pgfmathresult=1
        \node[smalldotnode] (n3\i) [right=of n3\prev] {};
      \else
        \node[squaresmalldotnode] (n3\i) [right=of n3\prev] {};
      \fi
    \fi
  \fi
}
\node[smalldotnode] (m31) [above=0.2cm of n314] {};
\node[squaresmalldotnode] (m32) [above=0.2cm of n316] {};
\draw (m31) -- (m32);
\draw (m31) -- (n312);
\draw (n318) -- (m32);

\node[smalldotnode] (p11) [above=0.2cm of n32] {};
\node[squaresmalldotnode] (p12) [above=0.2cm of n34] {};
\draw (p11) -- (p12);
\draw (p11) -- (n30);
\draw (n36) -- (p12);

\node[smalldotnode] (q11) [above=0.2cm of n323] {};
\node[squaresmalldotnode] (q12) [above=0.2cm of n325] {};
\draw (q11) -- (q12);
\draw (q11) -- (n321);
\draw (n327) -- (q12);

\node[squaresmalldotnode, fill=cyan] at (q12) {};
\node[squaresmalldotnode, fill=violet!50] at (p11) {};

\node[squaresmalldotnode, fill=violet] at (q11) {};
\node[squaresmalldotnode, fill=green!40] at (p12) {};

\foreach \i [evaluate=\i as \prev using int(\i-1)] in {1,...,30} {
  \draw (n3\prev) -- (n3\i);
}

\node[dotnode, fill=red] at (n30) {};
\node[smalldotnode, fill=lime] at (n31) {};
\node[squaresmalldotnode, fill=blue!50] at (n32) {};
\node[dotnode, fill=magenta] at (n33) {};
\node[smalldotnode, fill=blue] at (n34) {};
\node[squaresmalldotnode, fill=orange] at (n35) {};

\node[smalldotnode, fill=violet] at (n325) {};
\node[squaresmalldotnode, fill=cyan] at (n326) {};
\node[dotnode, fill=purple] at (n327) {};
\node[smalldotnode, fill=green] at (n328) {};
\node[squaresmalldotnode, fill=olive] at (n329) {};
\node[dotnode, fill=teal] at (n330) {};

\node[dotnode, label=below:{\scriptsize  $v_{-5}$}] (n40) [right=0.1cm of l4] {};
\foreach \i [evaluate=\i as \prev using int(\i-1),
             evaluate=\i as \labelval using int((\i - 15)/3),
             evaluate=\i as \smallcount using int((\i - 1) - floor((\i - 1)/3))] in {1,...,30} {
  \ifnum\i=1
    \node[smalldotnode] (n4\i) [right=of n40] {};
  \else
    \pgfmathparse{mod(\i,3)==0}
    \ifnum\pgfmathresult=1
      \node[dotnode, label=below:{\scriptsize $v_{\labelval}$}] (n4\i) [right=of n4\prev] {};
    \else
      \pgfmathparse{mod(\smallcount,2)==0}
      \ifnum\pgfmathresult=1
        \node[smalldotnode] (n4\i) [right=of n4\prev] {};
      \else
        \node[squaresmalldotnode] (n4\i) [right=of n4\prev] {};
      \fi
    \fi
  \fi
}
\node[smalldotnode] (m41) [above=0.2cm of n414] {};
\node[squaresmalldotnode] (m42) [above=0.2cm of n416] {};
\draw (m41) -- (m42);
\draw (m41) -- (n412);
\draw (n418) -- (m42);

\node[smalldotnode] (p11) [above=0.2cm of n42] {};
\node[squaresmalldotnode] (p12) [above=0.2cm of n44] {};
\draw (p11) -- (p12);
\draw (p11) -- (n40);
\draw (n46) -- (p12);

\node[smalldotnode] (q11) [above=0.2cm of n423] {};
\node[squaresmalldotnode] (q12) [above=0.2cm of n425] {};
\draw (q11) -- (q12);
\draw (q11) -- (n421);
\draw (n427) -- (q12);

\node[squaresmalldotnode, fill=cyan] at (q12) {};
\node[squaresmalldotnode, fill=violet!50] at (p11) {};

\node[squaresmalldotnode, fill=violet] at (q11) {};
\node[squaresmalldotnode, fill=green!40] at (p12) {};

\foreach \i [evaluate=\i as \prev using int(\i-1)] in {1,...,30} {
  \draw (n4\prev) -- (n4\i);
}

\node[dotnode, fill=red] at (n40) {};
\node[smalldotnode, fill=lime] at (n41) {};
\node[squaresmalldotnode, fill=blue!50] at (n42) {};
\node[dotnode, fill=magenta] at (n43) {};
\node[smalldotnode, fill=blue] at (n44) {};
\node[squaresmalldotnode, fill=orange] at (n45) {};
\node[dotnode, fill=red!50] at (n46) {};
\node[dotnode, fill=brown] at (n49) {};

\node[dotnode, fill=pink] at (n421) {};

\node[dotnode, fill=yellow] at (n424) {};
\node[smalldotnode, fill=violet] at (n425) {};
\node[squaresmalldotnode, fill=cyan] at (n426) {};
\node[dotnode, fill=purple] at (n427) {};
\node[smalldotnode, fill=green] at (n428) {};
\node[squaresmalldotnode, fill=olive] at (n429) {};
\node[dotnode, fill=teal] at (n430) {};

\end{tikzpicture}
    \caption{Application of $\WL_2$ to $H=\Gad(G)$, for  $G$ from \Cref{thm:LonotC2}; for  readability we draw  only 
    gadgetised edges corresponding to $v_i,v_{i+1}$ in $G$, as well as to edges $(v_{-5}, v_{-3})$, $(v_{-1}, v_{-1})$, and $(v_{2}, v_{4})$, which helps to  understand better the colourings}
    \label{fig::WL_LOGLO}
\end{figure}

By combining \Cref{gadlinisfo,gloisacrgnn}, 
we  obtain  a solution to the open problem of \citet{DBLP:conf/iclr/BarceloKM0RS20}.

\begin{corollary}
Over undirected graphs,
there are \FO{} node classifiers expressible by ACR-GNNs which are not expressible in $\Ctwo$.
In particular, $\varphi_{\GLO}(x)$ is such a classifier.
\end{corollary}

The above result,
shows that ACR-GNNs can express \FO{} node classifiers beyond $\Ctwo$.
Consequently, we establish that the converse of the result of~\citet[Theorem~5.1]{DBLP:conf/iclr/BarceloKM0RS20}
does not hold.
As we show in the following short section, our results have interesting implications beyond the expressive power of  GNNs, 
contributing to a better understanding of the expressiveness of logics.


\section{Impact on the Expressiveness of Logics}\label{sec:impact}

It turns out that our results can be used to show an interesting relation between the expressive power of finitary and infinitary logics.
To formulate this result, let us use $\LCtwo$ for an extension of $\Ctwo$ which allows for infinitary conjunctions and disjunctions.
Notice that the expressive power of $\LCtwo$ is not only beyond $\Ctwo$, but also beyond the whole $\FO$.
For example $\LCtwo$ allows us to express parity of a graph size using the infinite formula:
$$
\exists_{=2} x (x=x) \lor 
\exists_{=4} x (x=x) \lor
\exists_{=6} x (x=x) \lor  \dots
$$
which is well-known to be inexpressible in \FO{}---it can be  shown by a standard application of  Ehrenfeucht--Fraïssé games \cite{DBLP:books/sp/Libkin04}.

This naturally leads us to the question: what are the \FO{} properties expressible in $\LCtwo$?
It maybe tempting to  assume that those are exactly the properties expressible in $\Ctwo$.
In other words, that the (semantical) intersection of $\LCtwo$ and $\FO$ is exactly $\Ctwo$.
As we show next, it is not true.

\begin{theorem}
There are strictly more $\FO$ properties expressible in $\LCtwo$ than the properties expressible in $\Ctwo$.
This result holds both  over directed and undirected graphs.
\end{theorem}
\begin{proof}[Proof sketch]
Clearly each $\Ctwo$ property can be expressed in both $\FO$ and in $\LCtwo$.
Thus, it suffices to show properties which disprove  the opposite implication. For this, we can show that both $\vp_\LO(x)$ and $\vp_{\GLO}(x)$ are expressible in $\LCtwo$.
Indeed, by the results obtained in the paper it suffices to show that the third condition from \Cref{linear} can be expressed in $\LCtwo$ over directed graphs as 
$$
\bigwedge_{i\in \mathbb{N}} 
\forall x \forall y 
\left( 
\exists_{=i} y E(x,y)
\land
\exists_{=i} x E(y,x) 
 \to x = y \right)
$$
whereas $\psi$ from \Cref{gloisacrgnn} is expressed in $\LCtwo$   over undirected graphs as
$$
\begin{aligned}
&\bigwedge_{i\in \mathbb{N}}  \; \bigwedge_{  j \in \mathbb{N} : i<j} \Bigg[
  \exists_{j+1} x P_1(x) 
  \to  
 \exists x    \Bigg( \exists_{=j}y ( P_2(y)\land E(x,y))
  \\ 
  & \land P_1(x)\land \exists y \bigg( P_2(y) \wedge E(x,y) 
  \wedge 
  \exists x \Big( P_3(x) \wedge E(y,x) 
  \\ 
  & \wedge  \exists y \big(P_1(y) \wedge E(x,y) 
  \wedge 
  \exists_{=i} x (P_2(x) \wedge E(y,x)) \big) 
  \Big) 
  \bigg)  
  \Bigg) 
  \Bigg].
\end{aligned}
$$
Note that both formulas    rely on infinite conjunctions. 
\end{proof}

\section{Conclusions}

In this paper, we have solved the open problem  asking whether \FO{} classifiers expressible by aggregate-combine-readout GNNs are exactly the classifiers expressible in logic \Ctwo{} \cite{DBLP:conf/iclr/BarceloKM0RS20}.
As we show, the answer is negative. In particular, over both directed and undirected graphs, \FO{} classifiers expressible by ACR-GNNs have a strictly higher expressive power than \Ctwo{}.
Recall that the distinguishing power of AC-GNNs is the same as of the 1-dimensional Weisfeiler-Leman algorithm, and so, the same as of \Ctwo{}.
It turns out, however, that  the logical (\FO) expressive power of standard GNN architectures cannot be characterised by \Ctwo{}.
In particular, AC-GNNs can express strictly less $\FO$ properties  than $\Ctwo$, whereas ACR-GNNs can express strictly more  $\FO$ properties than $\Ctwo$.
Interestingly our results transfer to results on the expressive power of infinitary logics.
As we have shown, the infinitary version of \Ctwo{} can express strictly more \FO{} properties than the standard, finitary, $\Ctwo$.

\bibliography{GNN}


\input{appendix}

\end{document}

%% file: macros.tex
\usepackage[dvipsnames]{xcolor}

\usepackage{tikz}
\usetikzlibrary{shapes.geometric}

\usepackage{cancel}

\usepackage{adjustbox}

\usepackage{graphicx} 
\usepackage{amsthm}
\usepackage{amsmath}
\usepackage{amssymb}
\usepackage[T1]{fontenc}
\usepackage{todonotes}
\usepackage{graphicx}
\usepackage{enumitem}
\usepackage{xcolor}
\usepackage{cleveref}
\usepackage{tikz}
\usepackage{tikz-cd}
\usepackage{cleveref}
\usepackage{stmaryrd}
\usepackage{stix2}

\usetikzlibrary{positioning, arrows.meta}

\newcommand{\R}{\mathbb{R}}

\usepackage{thm-restate}
 \declaretheorem[name=Theorem]{thm}
\newtheorem{theorem}[thm]{Theorem}

\newtheorem*{example*}{Example}
\newtheorem{definition}[thm]{Definition}

\newtheorem{corollary}[thm]{Corollary}
\newtheorem*{claim*}{Claim}
\newtheorem*{theorem*}{Theorem}

\newcommand{\A}{\mathcal{OR}}
\newcommand{\G}{{G}}

\renewcommand{\L}{\mathcal{L}}

\newcommand{\M}{\mathfrak{M}}
\newcommand{\vp}{\varphi}

\newcommand{\sq}{\subseteq}

\newcommand{\Ctwo}{\ensuremath{\text{C}^2}}
\newcommand{\FOtwo}{\ensuremath{\text{FO}^2}}

\newcommand{\FO}{\ensuremath{\text{FO}}}

\newcommand{\LO}{\ensuremath{\mathit{Lin}}}
\newcommand{\GLO}{\ensuremath{\mathit{GadLin}}}
\newcommand{\Gad}{\ensuremath{\mathsf{gad}}}

\newcommand{\LCtwo}{\mathsf{inf}\text{-}\Ctwo}

\newcommand{\WL}{\textup{WL}}
\newcommand{\Ni}{\ensuremath{\overleftarrow{{N}}}}
\newcommand{\No}{\ensuremath{\overrightarrow{{N}}}}
\newcommand{\N}{\ensuremath{\mathcal{N}}}
\newcommand{\Lr}[1][]{%
  \ifx#1\empty
    \ensuremath{\mathcal{L}}%
  \else
    \ensuremath{\mathcal{L}^{#1}}%
  \fi
}

\newcommand{\true}{\ensuremath{\mathsf{true}}}
\newcommand{\false}{\ensuremath{\mathsf{false}}}

\newcommand{\comb}{\ensuremath{\mathsf{comb}}}
\newcommand{\agg}{\ensuremath{\mathsf{agg}}}

\newcommand{\aggi}{\ensuremath{\overleftarrow{\mathsf{agg}}}}
\newcommand{\aggo}{\ensuremath{\overrightarrow{\mathsf{agg}}}}

\newcommand{\readout}{\ensuremath{\mathsf{read}}}


%% file: appendix.tex
\onecolumn
\appendix

\section*{Technical Appendix}

Please note that we plan to simplify and polish some proofs in the appendix to further improve readibility.

\section{Proofs for \Cref{sec:WL}}

\normalform*

\begin{proof}
If the given formula does not have $x$ or $y$ as free variable, just consider the conjunction of the given formula with $x=x$ and $y=y$. This ensures that the given formula has both $x$ and $y$ as free variables. Thus we can denote the given formula as $\vp(x,y)$.

To obtain the required form, we transform $\varphi(x, y)$ into an equivalent $\Ctwo_{\ell,c}$ formula $\bigvee_{i=1}^n \bigwedge_{j=1}^{m_i} \psi_{i,j}$, where each $\psi_{i,j}$
has at most two free variables $x$ and $y$, and is either a literal (an atom or its negation), or starts with $\exists_k$, or starts with  $\neg \exists_k$. The process of constructing $\bigvee_{i=1}^n \bigwedge_{j=1}^{m_i} \psi_{i,j}$  is as follows.
Firstly, we write $\vp(x,y)$ in a form, where every negation is immediately followed by $\exists_k$ or by an atom. This is done by applying recursively De Morgan laws:
 $$\neg (a\land b) \equiv  \neg a\lor \neg b,$$
 $$\neg (a\lor b) \equiv \neg a\land \neg b.$$
 After this process, we arrived at a formula which is a positive Boolean combination (i.e. uses only disjunctions and conjunctions) of formulas which are literals, start with $\exists_k$, or with $\neg \exists_k$.
 Now, we apply distributivity of $\land$ over $\lor$, namely: 
 $$(a\lor b)\land c \equiv (a\land c)\lor (b\land c),$$
 to obtain the form  $\bigvee_{i=1}^n \bigwedge_{j=1}^{m_i} \psi_{i,j}$.

Next, we partition each $ \bigwedge_{j=1}^{m_i} \psi_{i,j}$ into three conjunctions: $\alpha_i(x)$ which is a conjunction of those $\psi_{i,j}$ that have  just $x$ as a free variable (if some conjunct $\psi_{i,j}$ has no free variables we  write it as $\psi_{i,j} \land (x=x)$), $\beta_i(y)$ which is a conjunction of those $\psi_{i,j}$ that have just $y$ as a free variable and $\gamma_i(x,y)$ that have both $x$ and $y$ as free variables.
We observe that no $\psi_{i,j}$ that is a conjunct of $\gamma_i(x,y)$ can start with $\exists_k$ or $\neg \exists_k$, as then the quantified variable would not be free in $\psi_{i,j}$, thus each has to be a literal, so $\gamma_i(x,y)$ is a conjunction of literals each having two free variables.

Hence $\vp(x,y)$ is  equivalent to
$\bigvee_{i=1}^n \big( \alpha_i(x) \land \beta_i(y) \land \gamma_i(x,y) \big)$. Formulas $\alpha_i(x)$ and $\beta_i(x)$ are as required by the lemma, so it remains  to show how to transform the formula to put each $\gamma_{i}(x, y)$ to a desired form.
Recall that $\gamma_i(x,y)$ is a conjunction of literals each having two free variables.
Six of such atoms exist:
$$E(x,y),E(y,x),x=y,$$ 
$$\neg E(x,y),\neg E(y,x),x\neq y,$$

Conjunction of any subset of above is  equivalent to a disjunction of a non-empty subset of  the following (all combinations of (negated) atoms from the set above, and $\bot$):
\begin{enumerate}[label=\arabic*.]
\item $\bot$
\item $E(x, y) \land E(y, x) \land x = y$
\item $\neg E(x, y) \land E(y, x) \land x = y$
\item $E(x, y) \land \neg E(y, x) \land x =y$
\item $E(x, y) \land E(y, x) \land x \neq y$
\item  $\neg E(x, y) \land \neg E(y, x) \land x = y$
\item $\neg E(x, y) \land E(y, x) \land x \neq y$
\item $E(x, y) \land \neg E(y, x) \land x \neq y$
\item $\neg E(x, y) \land \neg E(y, x) \land x \neq y$
\end{enumerate} 

Since we are considering simple graphs, only Formulas 5--9 are satisfiable. Moreover, over simple graphs they are equivalent to the following:
\begin{enumerate}
\item[5'.] $E(x, y) \land E(y, x) \land x\neq y$
\item[6'.] $x=y$
\item[7'.] $\neg E(x, y) \land E(y, x)$
\item[8'.] $E(x,y)\land \neg E(y,x)$
\item[9'.] $\neg E(x, y) \land \neg E(y, x) \land x \neq y$
\end{enumerate} 


Hence $\vp(x,y)$ is  equivalent to a formula of the form
$\bigvee_{i=1}^n \big( \alpha_i(x) \land \beta_i(y) \land \gamma_i(x,y) \big)$, where each $\gamma_{i}(x, y)$ is a disjunction of some of the Formulas 5'.--9' or it is $\bot$.
Then, we apply  exhaustively the following equivalence-preserving transformation: 
$$a\land b\land (c\lor c') \equiv  (a\land b\land c)\lor (a\land b\land c'),$$ 
to obtain  a formula of the form
$\bigvee_{i=1}^{n'} \big( \alpha_i(x) \land \beta_i(y) \land \gamma_i(x,y) \big)$, where each $\gamma_{i}(x, y)$ is equal to some of the Formulas 5'.--9' or is $\bot$.

Now remove the disjuncts $\alpha_i(x) \land \beta_i(y) \land \gamma_i(x,y)$, where $\gamma_i(x,y)$ is equal to $\bot$, because 
$$d\lor (a\land b\land \bot) \equiv d$$
for any formula $d$. 

Thus we get that $\vp(x,y)$ is  equivalent to a formula of the form
$\bigvee_{i=1}^{n''} \big( \alpha_i(x) \land \beta_i(y) \land \gamma_i(x,y) \big)$, where each $\gamma_{i}(x, y)$ is equal to some of the Formulas 5'.--9'. as required by the lemma, but this disjunction is possibly empty, i.e. $n''=0$.
If $n''=0$, then that $\phi(x,y)\equiv \bot$, so we write $\phi(x,y)$  as $ (x\neq x\land y\neq y\land x=y))$, which is in the correct form.
\end{proof}

\WLtheorem*
\begin{proof}
We show the equivalence by induction on $i \leq  \ell$.
In the base case, since graphs are simple, we have $G,u \equiv_{\Ctwo_{0, c}} H,v$ if and only if $u$ and $v$ satisfy the same unary predicates, which is equivalent to $W^0_c(u)=W^0_c(v)$.
In the inductive step we assume that the equivalence holds for some $i$, and we show separately each implication for $i+1$.

\bigskip 

For the forward implication, assume  that $ W^{i +1}_c(u) \neq W^{i +1}_c(v) $. 
We will construct a  $\Ctwo_{{i +1},c}$ formula $\varphi(x)$ such that \( G \models \varphi(u) \), but \( H \not\models \varphi (v) \).
We start the construction by defining formulas $\psi_t^i(x)$ for every colour $t$, with
$t=W^i_c(w)$ for some node $w$ in $G$ or $H$ which will be later shown to capture the properties of nodes that have colour $t$ in the $i$th iteration of $W_c$.
To this end, we let $\psi_t^i(x)$ be the conjunction of all
$\Ctwo_{{i},c}$
formulas $\psi(x)$ such that $G \models  \psi(w)$ or $H \models  \psi(w)$, for some $w$ with $W_c^i(w)=t$.
Note that,
up to the logical equivalence, there are  finitely many  \( \Ctwo_{i,c}\) formulas   \cite[Lemma 4.4]{DBLP:journals/combinatorica/CaiFI92}, so $\psi_t^i(x)$ is finite.

Now, we will construct $\varphi(x)$ using $\psi_t^i(x)$. Since $W^{i +1}_c(u) \neq W^{i +1}_c(v) $, by 
Equation~\eqref{eq:WLupdate} we have that (i) there is a colour $t$ such that 
$t = W^{i }_c(u) \neq W^{i }_c(v) $, or (ii) there  is a colour $t$ and $j \in \{1, \dots, 4 \}$ such that $t$  occurs $k \leq c$ times in the $j$th of the four multisets defining $W_c^i(u)$ in Equation~\eqref{eq:WLupdate},  and $k' \leq c$ times in the $j$th multiset defining $W^{i +1}_c(v)$, where $k \neq k'$. 
If Condition (i) holds, we let
$$
\varphi(x) = \psi_t^i(x).
$$
If Condition (ii) holds, without loss of generality assume that $k' < k$, and we let
$$\varphi(x) = \exists_k y (\psi_t^i(y) \land \chi_j(x,y)),$$ 
where 
$\chi_1(x, y) = E(x, y) \land E(y, x)$,
$\chi_2(x, y) = \neg E(x, y) \land E(y, x)$,
$\chi_3(x,y) = E(x, y) \land \neg E(y, x)$,
and
$\chi_4(x, y) = \neg E(x, y) \land \neg E(y, x) \land x \neq y$.
Note that $\varphi(x)$ is a $\Ctwo_{{i +1},c}$ formula.
Moreover, formulas $\chi_j$ correspond to the sets over which multisets $j$ are defined.
Hence, to show that \( G \models \varphi(u) \) and \( H \not\models \varphi (v) \), it remains to show that $\psi_t^i$ has the intended meaning, that is, for any node $w$ in $F \in \{G,H \}$, it holds that
$W^i_c(w)=t$ if and only if $F\models \psi_t^i(w)$.

Now, we will show the above equivalence. Let $ w $ be a node in $ F \in \{G, H\} $,  such that $ W^i_c(w) = t $. Let $\psi(x)\in \Ctwo_{i,c}$ be such that $ F_0 \models \psi(w_0) $ for some $w_0 $ in $F_0\in \{G,H\}$ with $W^i_c(w_0) = t$. 
By the definition of $\psi_t^i$, we need to show that $ F \models \psi(w) $. 
Since $ W^i_c(w) = W^i_c(w_0) $, by the inductive hypothesis, $ F \models \psi(w) $.  Hence $F\models \psi_t^i(w)$.
Now, assume that  $w$ is a node of  $F \in \{G, H\}$ such that $W^i_c(w) \neq t$.
Let $w_0$ be some node in $F_0 \in \{G, H\}$ such that $W^i_c(w_0) = t$. Since $W^i_c(w) \neq W^i_c(w_0)$, by the induction hypothesis, we have
$
W^i_c(w) \not\equiv_{\Ctwo_{i,c}} W^i_c(w_0).
$
Thus, there exists a formula $\psi(x)\in \Ctwo_{i,c}$ such that $F_0 \models \psi(w_0)$ but $F \not\models \psi(w)$. 
By the definition, $\psi(x)$ is conjunct of $\psi_t^i(x)$, so $F \not\models \psi_t^i(w)$, as required. 



\bigskip 

Next, we will show the opposite implication from the inductive step. 
Assume that \( W^{i+1}_c(u) = W^{i+1}_c(v) \). By induction on the structure of \( \varphi(x) \in \Ctwo_{{i+1},c}\),
we will show that
$
G \models \varphi(u)$ if and only if $H \models \varphi(v)$.
If  $\varphi(x)$ is atomic, it suffices to observe that 
$ W^{i+1}_c(u) = W^{i+1}_c(v)$ implies $ W^{0}_c(u) = W^{0}_c(v)$, so $G\models \vp(u)$ if and only if $H\models \vp(v)$.
%
If $\vp(x) = \neg \psi(x)$, then since  $
G \models  \psi(u)$ if and only if $H \models \psi(v)$, we get $
G \models \varphi(u)$ if and only if $H \models \varphi(v)$.
If  $\vp(x)$ is a conjunction, an analogous simple argument guarantees that 
$G \models \varphi(u)$ if and only if $H \models \varphi(v)$.
It remains to consider 
$\varphi(x) = \exists_{k} y \psi(x, y)$,
where \( \psi(x, y) \in \Ctwo_{i,c} \) and \( k \leq c \). 
At first we will show that it holds for formulas $\psi(x,y)$ which are of the form $\chi_j(x,y) \land \eta(y)$, where $\chi_j(x,y)$ is one of the four formulas $\chi_1(x,y), \dots, \chi_4(x,y)$ defined already in this proof, and $\eta(y) \in  \Ctwo_{i,c}$.
Then we will use  \Cref{twovarstructure} to generalise this result to any $\psi(x,y) \in \Ctwo_{i,c} $.

Assume that $G \models \exists_k y ( \chi_j(u,y) \land \eta(y)) $, for some $\chi_j(u,y) \land \eta(y)$ described above and and $k\leqslant c$(?).
Since $\eta(y)\in \Ctwo_{i,c}$, by the inductive hypothesis  there is a  set  $T_\eta^i$
of colours in the $i$th iteration of $WL_c$  corresponding to nodes satisfying $\eta(x$), namely $T_\eta^i$  is such that for any $F\in\{G,H\}$ and any node $w$ in $F$, we have  $F \models \eta(w)$ if and only if $W^i_c(w)\in T_\eta^i$. 
Hence $G,u \models \exists_k y ( \chi_j(u,y) \land \eta(y))$, by the form of $\chi_j$, implies that the $j$th of the four multisets defining $W_c^{i+1}(u)$ in Equation~\eqref{eq:WLupdate} has at least $k$ occurrences of colours from the set $T_\eta^i$.
Since $W_c^{i+1}(u) = W_c^{i+1}(v)$ and $k\leqslant c$, the $j$th multiset of $W_c^{i+1}(v)$ also contains at least $k$ occurrences of colours from the set $T_\eta^i$.
Hence $H,v \models \exists_k y (\chi_j(v,y) \land \eta(y)) $.
The other direction is proved analogously, so 
 $G,u \models \exists_k y ( \chi_j(u,y) \land \eta(y)) $ if and only if
$H,v \models \exists_k y (\chi_j(v,y) \land \eta(y)) $.

We will now consider the general case, so let $\psi(x,y)\in \Ctwo_{i,c}$ be any formula with $G\models \exists_k y\psi(u,y)$, where $k\leqslant c$. We need to show that $H\models \exists_k y\psi(v,y)$. We show that for $j\in \{1,2,3,4\}$ we have that if $G\models \exists_{k_j} y(\chi_j(u,y)\land \psi(u,y)),$ then $H\models \exists k_j y( \chi_j(v,y)\land \psi(v,y))$.

This is sufficient for the following reason: choose $k_j$ to be the maximal number in $\{1,\ldots, c\}$ with $G\models \exists_{k_j} y(\chi_j(u,y)\land \psi(u,y))$ and $k_0=1$ if $G\models \psi(u,u)$ and $k_0=0$ otherwise. 
By maximality of each $k_j$, the choice of $k_0$, the fact that there are at least $k$ nodes $w$ with $G\models \psi(u,w)$ and the fact that for every node $w$ of $G$ either $v=w$ or $\chi_j(v,w)$ for some $j\in \{1,2,3,4\}$, we must have $\sum_{j=0}^4k_j\geqslant k$. 
But recall that we also have $H\models \exists_{k_j} y(\chi_j(v,y)\land \psi(v,y))$ and by the inductive hypothesis $G\models \psi(u,u)$ iff $H\models \psi(v,v)$, so if $k_0=1$, then $H\models\psi(v,v)$ and $H\not\models \psi(v,v)$ otherwise. Combining this with the fact that there is no node $w$ of $H$ for which at least two of the formulas $\chi_j$ are satisfied and the fact that $\chi_j(v,v)$ is never satisfied, we obtain that there are at least $\sum_{j=0}^4 k_j$ distinct nodes $w$ with $H\models \psi(u,w)$, so because $\sum_{j=0}^4 k_j\geqslant k$ $H\models \exists k y \psi(u,y)$, as required. 

It remains to show that for any $\psi(x,y)\in \Ctwo_{i,c}$ and for $j\in \{1,2,3,4\}$ we have that if $G\models \exists_{k_j} y(\chi_j(u,y)\land \psi(u,y)),$ then $H\models \exists_{k_j} y(\chi_j(v,y)\land \psi(v,y))$. By Lemma~\ref{twovarstructure}, write \( \psi(x, y) \) as a disjunction
\[
\bigvee_{s=1}^n\alpha_s(x) \land \beta_s(y) \land \gamma_s(x, y),
\]
where \( \alpha_s(x), \beta_s(y) \in \Ctwo_{i,c} \) and $\gamma_s(x,y)$ is one of the following five formulas: $\chi_1(x,y),\chi_2(x,y),\chi_3(x,y),\chi_4(x,y),$ and $x=y$.

Let $S\sq \{1,\ldots, n\}$ be the set of indices $s$ for which $\gamma_s(x,y)=\chi_j(x,y)$ and $G\models \alpha_s(u)$. Define\footnote{By convention, the empty disjunction is defined to be $\bot$.} $$\eta(y):=\bigvee_{s\in S}\beta_s(y),\quad \text{so}\quad \eta(y)\in \Ctwo_{i,c}.$$

We will now show that for any node $w$, we have $G,w \models \chi_j(u,w)\land \eta(w)$ if and only if $G,w\models \chi_j(u,w)\land\psi(u, w)$. Indeed, $G ,w\models\chi_j(u,w)\land \eta(w)$ if and only if $G,w \models (\chi_j(u,w)\land \beta_s(w))$ for some $s\in S$. But recall that there is no node $w$ of $G$ for which at least two of the formulas $\chi_j$ are satisfied and $\chi_j(v,v)$ is never satisfied, so the latter happens if and only if $G,w\models \gamma_s(u,w)\land \beta_s(w)$ for some $s$ with $G\models \alpha_s(u)$ and $\gamma_s(u,w)=\chi_j(x,y)$, which happens if and only if $G,w\models \alpha_s(u)\land \beta_s(w)\land \gamma_s(u,w)$ for some $s$, which is equivalent to $G,w\models \psi(u,w)$, as required.

By the inductive hypothesis, since $W^i_c(u) = W^i_c(v)$, we have: \[ G \models \alpha_j(u) \Leftrightarrow H \models \alpha_j(v) \quad \text{for each } j, \] so for any node $w$: $H,w \models \chi_j(u,w)\land \eta(w)$ if and only if $H,w\models \chi_j(u,w)\land\psi(u, w)$. Indeed, that is because construction of $\eta(y)$ with respect to conditions $G\models \alpha_s(u)$ or $H\models \alpha_s(v)$ yield the same formula, as this is an equivalent condition.

Finally, since $G\models \exists_{k_j} y(\chi_j(u,y)\land \psi(u,y)),$ so $G\models \exists_{k_j} y(\chi_j(u,y)\land \eta(y)),$ so $H\models \exists_{k_j} y(\chi_j(v,y)\land \eta(y)),$ so $H\models \exists_{k_j} y(\chi_j(v,y)\land \psi(v,y))$, which completes the proof.
\end{proof}

\newpage
\section{Proofs for \Cref{sec:directed}}

\linearalternative*
\begin{proof}
Clearly, every strict linear order satisfies the three properties from the proposition.
Below we  show the opposite direction.
We know that $E$ is irreflexive and total, so it remains to show that $E$ is transitive.
Assume that there are $n$ elements.
Since each element has a different number of $E$-successors and $E$ is irreflexive, 
we can call the elements $v_0, \dots, v_{n-1}$, where $v_i$ is the unique element whose number of $E$-successors is  $i$.
To show that $E$ is transitive, we will prove a more general statement, namely that for all $v_i$ and all $v_j$ we have $(v_i,v_j) \in E$ if and only if $i > j$ (which  implies the transitivity of $E$).

We show the statement by a strong induction on $i$.
In the basis we have $i=0$. Then both implications in the statement hold trivially since  $v_0$ has no $E$-successors and there is no $v_j$ with $j<0$.
For the inductive step, assume that the equivalence holds for all numbers smaller than $i$; we will show that it holds for $i$.
We fix an arbitrary $v_j$ and consider two cases: $i>j$ and $j \geq i$. 
If $i>j$, we need to show that $(v_i,v_j) \in E$.
By the inductive assumption, $j \not > i$ implies that $(v_j,v_i) \notin E$. 
Since $E$ is total, we need to have $(v_i,v_j) \in E$, as required.
If $j \geq i$,  we need to show that $(v_i,v_j) \notin E$.
As we have showed in the first case, we have $(v_i,v_j) \in E$ for all $j$ with  $i>j$.
If we had additionally $(v_i,v_j) \in E$ for some $j \geq i$, then $v_i$ would have more than $i$ $E$-successors, which raises a contradiction. Therefore $(v_i,v_j) \notin E$.
\end{proof}

\GNNforLin*
\begin{proof}
We will construct an ACR-GNN  $\N$; its application to a linear order of length four is presented in \Cref{fig::GNN_LO}.
The first layer maps the initial vector of a node $v$ into the number $10^n$, where $n$ is the in-degree of $v$. This is obtained by setting  $\aggi(M)=  10^{|M|}$ and $\comb(x,y) =y$.
The second layer  maps a vector of  $v$ into a vector in $\R^2$ of the form $(10^n, 10^{k_1} + \dots + 10^{k_n} )$ where $10^n$ is as in the first layer, whereas each $k_i$ is the in-degree of the $i$th among the $n$ in-neighbours of $v$. This is obtained by setting  $\aggi(M)= sum (M)$ and $\comb(x,y) = (x,y)$.
The third layers  maps each vector into 1 or 0 by setting
$\readout(M) = 1$ if both of the following conditions hold:
\begin{enumerate}
\item[(i)] $x[1] \neq y[1]$, for all $x,y \in M$ with $x\neq y$.
\item[(ii)] if $x[1] = 10^n$, then $x[2] =
\underbrace{1\ldots1}_{n \text{ times}}
$, for each $x \in M$ (i.e. $\frac{x[1] - 1}{9} =  x[2]$).
\end{enumerate}
If some of these conditions does not hold, we set $\readout(M) =0$.
Finally, we let
$\comb(x,y)=y$.

We claim that for any graph $G=(V,E,\lambda)$, if $E$ is a strict linear order, then  $\N(G,v)=1$ and otherwise  $\N(G,v)=0$, for any node $v$ in $G$.
First, assume that $E$ is a strict linear order.
Hence, all nodes have different in-degrees, namely their in-degrees are $0, \dots,|V|-1$.
Moreover if a node has an in-degree $n$, then its immediate predecessor has in-degree $n-1$.
Thus, by the construction of $\N$, after applying first two layers, each node $v$ has an embedding $(10^n,
\underbrace{1\ldots1}_{n \text{ times}}
)$, for $n$ being the in-degree of $v$.
Hence, both Conditions (i) and (ii) hold, and  so $\N(G,v)=1$ for all nodes $v$ in $G$.

For the opposite direction assume that 
$\N(G,v)=1$.
By
Condition (i), each node in $G$ has a different in-degree.
By Condition (ii), the edge relation cannot have loops; otherwise we had $x[2] \geq x[1]$ for some $x$, which is forbidden by Condition (ii).
Finally, Condition (ii) also implies that every node of in-degree $n$ has incoming edges from all nodes with in-degree smaller than $n$. Since every node has a different in-degree, it follows that the relation is total.
Hence, by  \Cref{linear},
 $E$ is a strict linear order.
\end{proof}

\WLforLin*
\begin{proof}
Suppose towards a  contradiction that $\varphi_{\LO}(x)$ is expressible in $\Ctwo$, so it is definable by a formula in $\Ctwo_{\ell, c}$, for some $\ell$ and $c$.
To obtain a contradiction, we will construct a graph $G$ with nodes $v_i$ and a graph $G'$ with corresponding nodes $v_i'$, such that $G \models \varphi_{\LO}(v_i)$ and $G' \not\models \varphi_{\LO}(v_i')$, but $G,v_i \equiv_{\Ctwo_{l,c}} G',v_i$ for all nodes $v_i$.

Let $n = \ell \cdot c +1$.
We define $G=(V,E,\lambda)$ as a strict linear order over $2n+1$ nodes $V=\{v_{-n}, \dots, v_n \}$, with  $E =\{ (v_i,v_j) : i < j \}$, and $\lambda(v_i)=0$ for each $v_i$.
We let $G'=(V',E',\lambda')$ be such that 
$V'=\{v'_{-n}, \dots, v'_n \}$, $E' =\{ (v'_i,v'_j) : i < j \} \setminus \{ (v'_{-1},v'_{1}) \} \cup \{ (v'_{1},v'_{-1}) \}$, and $\lambda'(v_i')=0$ for each $v_i'$.
For example, if $c=2$ and $\ell =2$, the graphs $G$ is  depicted on top of \Cref{fig::WL_LO}; graph $G'$ is similar, but instead of $(v'_{-1},v'_{1})$ it has the opposite edge $(v'_{1},v'_{-1})$.
Notice that both graphs are irreflexive, asymmetric, and total, but only $G$ is transitive.
Hence, for all nodes $v_i$, we have $G \models \varphi_{\LO}(v_i)$ and $G' \not\models \varphi_{\LO}(v_i')$, so it remains to show that $G,v_i \equiv_{\Ctwo_{l,c}} G',v'_i$.
To this end, by \Cref{mainWLtheorem},
 it suffices to show that $W^\ell_c (v_i) = W^\ell_c (v_i')$.
We will prove it by showing, with a simultaneous induction on $k \leq \ell$,  the  following  two statements:
\begin{enumerate}[label=(\roman*), leftmargin=*, align=left, labelsep=-0.5em]
\item
$W^k_c (v_i) = W^k_c (v_i')$, for  $i \in \{-n, \dots, n\}$,

\item
$W^k_c (v_i) = W^k_c (v_j)$, for  $i,j \in \{ -(n - ck), \dots,  n -ck   \}$.
\end{enumerate}
In the base of the induction, for $k=0$, both Statements~(i) and (ii) hold, since $W^0_c (v_i) = W^0_c (v_i')=0$. 
For the inductive step, assume that Statements (i) and (ii) hold for some $k<\ell$.
We will show that both statements  hold  for $k+1$.

We start by showing Statement (ii). 
Let us fix any $i,j \in \{ -(n - c(k+1)), \dots,  n -c(k+1)   \}$.
We need to show that  $W^{k+1}_c (v_i) = W^{k+1}_c (v_j)$.
By \Cref{eq:WLupdate} together with the fact that both $E$ and $E'$ are irreflexive, asymmetric, and total, it suffices to show the following three equalities:
\begin{itemize}
\item[(1)] 
$W^{k}_c (v_i) = W^{k}_c (v_j)$, 
\item[(2)] 
$\lBrace W^{k}_c (v_r) : v_r \in \Ni_G(v_i)  \rBrace^c = \lBrace W^{k}_c (v_r) : v_r \in \Ni_G(v_j)  \rBrace^c $,
\item[(3)] 
$\lBrace W^{k}_c (v_r) : v_r \in \No_G(v_i)  \rBrace^c = \lBrace W^{k}_c (v_r) : v_r \in \No_G(v_j)  \rBrace^c $.
\end{itemize}
Equality (1) holds by the inductive assumption for Statement (ii).
To show Equalities (2) and (3), we let
 $S = \{ -(n-ck), \dots, n-ck  \}$. 
We will show two versions of each equality: for multisets with  $r \notin S$ and with $r \in S$ (which is stronger  than original equalities considering all $r$).
For $r \notin S$, by the form of $G$, we have 
$v_r \in \Ni_G(v_i)$ iff $v_r \in \Ni_G(v_j)$, and
$v_r \in \No_G(v_i)$ iff $v_r \in \No_G(v_j)$, so
 Equalities (2) and (3)  hold. 
Now consider multisets with $r \in S$. By the inductive assumption for Statement (ii),
$W^{k}_c (v_r)$ is the same for all $r \in S$.
So to prove Equality (2), it suffices to show that both $v_i$ and $v_j$ have at least $c$ many in-neighbours $v_r$ with $r \in S$.
For this, recall that 
$i,j \in \{ - (n-c(k+1)), \dots, n-c(k+1) \}$,
so for each $r\in \{-(n-ck),\ldots, -(n-c(k+1)-1$  we have both $(v_r,v_i) \in E$ and $(v_r,v_j) \in E$.
Note that there are exactly $c$ such nodes $v_r$, which finishes the proof of Equality (2).
Equality (3) for $r \in S$ is showed analogously, so Statement (ii) holds. 

Next, we show the inductive step for Statement (i).
We start by observing that $W^{k+1}_c (v_i) = W^{k+1}_c (v_i')$ for $i \notin \{-1, 1 \}$, which follows from the inductive assumption for  Statement (i) together with the fact that  $v_i$ and $v_i'$ have the same $E$-successors and $E$-predecessors (modulo priming of symbols).
It remains to show Statement (i) for $i \in \{-1,1 \}$.
Note that we have $W^{k+1}_c (v_0) = W^{k+1}_c (v_0')$.
By the inductive step for Statement (ii)
we obtain that $W^{k+1}_c (v_{-1}) = W^{k+1}_c (v_0) = W^{k+1}_c (v_{1})$.
Although we have showed Statement (ii) for $G$, the same argumentation can be used for $G'$,  so  $W^{k+1}_c (v_{-1}') = W^{k+1}_c (v_0') = W^{k+1}_c (v_{1}')$.
Thus, 
$W^{k+1}_c (v_{i}) = W^{k+1}_c (v_{i}')$  for $i \in \{-1, 1 \}$.
\end{proof}

\newpage
\section{Proofs for \Cref{sec:undirected}}

\FOgadget*
\begin{proof}
We will express $\vp_\GLO(x)$ as a conjunction of four $\FO$ formulas $\vp_1$, $\vp_2$, $\vp_3$, and $\vp_4$.
Recall that we identify graphs with $\FO$ structures interpreting unary predicates $P_1, \dots , P_d$, where $d$ is the dimension of the graph, and one binary predicate $E$.
Since gadgetisations are always of dimension $d=3$, our formulas will mention three unary predicated $P_1$, $P_2$, and $P_3$.

Formula $\vp_1$ states that 
$P_1$, $P_2$, and $P_3$
partition the set of nodes.
Formula $\vp_2$ states that every node satisfying $P_2$ has exactly two $E$-neighbours: one satisfying $P_1$ and the other satisfying $P_3$.
It states also that every node satisfying $P_3$ has exactly two $E$-neighbours: one satisfying $P_1$ and the other satisfying $P_2$.
Finally, it states that 
if $u$ and $v$ are nodes satisfying $P_1$, then $E(u,v)$ does not hold.
Formulas $\vp_3$ and $\vp_4$ are about \emph{gadgetised edges}, which are paths in $\Gad(G)$ that correspond to directed edges in $G$.
In particular, we let a gadgetised edge be a path of the form $E(u,w)$, $E(w,v)$, $E(v,z)$ for which it holds that $P_1(u)$ and $P_1(z)$, and either $P_2(b)$ and $P_3(v)$, or $P_3(v)$ and $P_2(w)$.
We will say that such a gadgetised edge is from $u$ to $z$ (notice that direction plays a crucial role in  gadgetised edges).
Formula $\vp_3$ states that between any two distinct nodes satisfying $P_1$ there is exactly one gadgetised edge.
Formula $\vp_4$, in turn, states that there are no nodes $u,w,v$  with gadgetised edges from $v$ to $w$, from $w$ to $u$ and from $u$ to $v$.

Next, we show how to express 
$\vp_1$--$\vp_4$ in $\FO$.
This is done as follows, where $\oplus$ stands for the XOR connective:
\begin{align*}
\vp_1 & =\forall x \big((P_1(x) \lor P_2(x) \lor P_3(x)) \land \neg(P_1(x) \land P_2(x)) \land \neg(P_1(x) \land P_3(x)) \land \neg(P_2(x) \land P_3(x)) \big)
\\
\vp_2 & = \forall x ( P_2(x) \to \exists_{=2} y E(x,y) \land   
\exists_{=1} y (E(x,y) \land P_1(y)) 
\land 
\exists_{=1} y (E(x,y) \land P_3(y) ) ) \land{} 
\\
& \ \quad \forall x ( P_3(x) \to \exists_{=2} y E(x,y) \land   
\exists_{=1} y (E(x,y) \land P_1(y)) 
\land 
\exists_{=1} y (E(x,y) \land P_2(y) ) ) \land{} 
\\
& \ \quad  \forall x \forall y \neg ( P_1(x) \land P_1(y) \land E(x,y) )
\\
\vp_3 & =
\forall x \forall x' \big( (P_1(x) \land P_1(x') \land x \neq x') \to\\& \ \quad
\big( 
\exists_{=1} y \exists_{=1} z \big( P_2(y) \land P_3(z) \land E(x,y) \land E(y,z) \land E(z,x') \big) 
\oplus{} 
\\
&\ \quad\exists_{=1} y \exists_{=1} z \big( P_3(y) \land P_2(z) \land E(x,y) \land E(y,z) \land E(z,x') \big) 
\big) \big)
\\
\vp_4 & =
\neg \exists x_1 \exists x_2 \exists x_3 \exists y_1 \exists y_2 \exists y_3 \exists z_1 \exists z_2 \exists z_3  \big(
P_1(x_1) \land P_1(x_2) \land P_1(x_3) \land P_2(y_1) \land P_2(y_2) \land P_2(y_3) \land P_3(z_1) \land P_3(z_2) \land P_3(z_3) \land{} \\
& \ \quad E(x_1,y_1) \land E(y_1,z_1) \land E(z_1,x_2) \land E(x_2,y_2) \land E(y_2,z_2) \land E(z_2,x_3) \land E(x_3,y_3) \land E(y_3,z_3) \land E(z_3,x_1)
\big)
\end{align*}

We claim  that over undirected graphs, $\vp_\GLO(x)$ is equivalent to $\vp_1 \land \vp_2 \land \vp_3 \land \vp_4 \land (x=x)$.
For the forward implication assume that $G' \models \vp_\GLO(v)$, so $G'$ is isomorphic to $\Gad(G)$, for some strict linear order $G$.
Directly by the definition of $\Gad(G)$, we obtain that $\Gad(G) \models \vp_1$ and 
$\Gad(G) \models \vp_2$.
Since $G$ is total and asymmetric, $G$ has exactly one edge between any two distinct nodes. Hence, there is exactly one gadgetised edge between any two distinct nodes of  $\Gad(G)$ satisfying $P_1$, and so, $\Gad(G) \models \vp_3$. 
Moreover, as
$G$ is a strict linear order, it cannot have a cycle. 
In particular $G$ has no cycle of length 3, so $\Gad(G)$ has no cycle of length 3 over gadgetised edges, and so,  $\Gad(G) \models \vp_4$.

 For the opposite direction assume that an  undirected graph $G=(V,E,\lambda)$ of dimension 3 satisfies $\vp_1$, $\vp_2$, $\vp_3$, and $\vp_4$.
We define a directed graph $G' = (V',E',\lambda')$ such that
\begin{itemize}
\item $V' = \{ v\in V : G \models P_1(v) \}$,

\item $E' = \{ (u,v) \in V' \times V' : \text{ there is a gadgetised edge from $u$ to $v$ in $G$ }  \} $,

\item $\lambda'(v)=0$, for all $v \in V'$.
\end{itemize}
It suffices to show that  $G'$ is a strict linear order and that $G$ is isomorphic to $\Gad (G')$. 

To show that $G'$ is a strict linear order, we will show that $E'$ is total,   irreflexive, and transitive.
To show that $E'$ is total, fix $u,v \in V'$ such that $u \neq v$.
By the construction of $V'$, we have $G \models P_1(u)$ and $G \models P_1(v)$.
Since $G\models \vp_3$,  there is a gadgetised edge in $G$ from $u$ to $v$ or from $v$ to $u$.
Hence, by the definition of $E'$ we have $(u,v) \in E'$ or $(v,u) \in E'$, and so, $E'$ is total.
To show that $E'$ is irreflexive,
suppose that  $(v,v) \in E'$ for some $v \in V'$.
Hence, there is a gadgetised edge from $v$ to $v$ in $G$.
This, however contradicts $G \models \vp_4$, so $E'$ must be irreflexive.
To show that $E'$ is transitive, suppose that  $(u,v) \in E'$ and $(v,w) \in E'$, but $(u,w) \notin E'$.
By totality of $E'$, we have $(w,u) \in E'$ or $u=w$.
If $(w,u) \in E'$, then $G \models \vp_4$ raises a contradiction.
If $u=w$ then $G \models \vp_3$ raises a contradiction.
Hence $E'$ must be transitive.

To prove that $G=(V,E,\lambda)$ is isomorphic to $\Gad (G')$, we define function $f$ mapping nodes of $V$ to nodes of $\Gad (G')$ as follows. For each $u \in V$:
$$
f(u):=\begin{cases}
v_u^1, & \text{if } G\models P_1(u), 
\\
v_{(u',w')}^2, & \text{if } G\models P_2(u), \text{ for $u',w'\in V'$ the unique nodes with }G\models \exists w \big( P_3(w)\land E(u',u)\land E(u,w)\land E(w,w') \big), 
\\
v_{(u',w')}^3, &\text{if } G\models P_3(u), \text{ for $u',w'\in V'$  the unique nodes with }G\models \exists w \big( P_2(w)\land E(u',w)\land E(w,u)\land E(u,w') \big).
\end{cases}
$$
It remains to show that $f$ is well-defined, bijective, and that $(u,v) \in E$ if and only if $(f(u), f(v))\in E'$ and for $u\in V$ and $i\in \{1,2,3\}$, $G\models P_i(u)$ if and only if $\Gad(G')\models P_i(f(u))$.

To show that $f$ is a well-defined function, we need to show that every $ u \in V $ satisfies exactly one of the three cases in the definition of $f$. 
As $G \models \vp_1$, each $u\in V$ satisfies exactly one of $P_1$, $P_2$, and $P_3$, so every $u\in V$ satisfies at most one of the three cases in the definition. To show that at least one of the conditions is satisfied, it remains to show that if $G\models P_2(u)$ or $G\models P_3(u)$, then there exist unique $u',w'$ satisfying the respective conditions.
As $G \models \vp_2$, for every node $u\in V$ with $G\models P_2(u)$,
node $u$ belongs to a unique gadgetised edge, whose endpoints, say $u'$ and $w'$ are the unique points under consideration, which guarantees that $f(u)$ is well defined in this case.
The case of $G\models P_3(u)$, follows from an  analogous argument. Thus $f$ is well-defined.

To show that $f$ is injective, suppose $f(u)=f(u'')$, where $u,u''\in V$. There are three cases to consider. If $G\models P_1(u)$, then $G\models P_1(u'')$, so $v^1_{u}=v^1_{u''},$ so $u=u''$. If $G\models P_2(u)$, then $G\models P_2(u'')$, so $f(u)=f(u'')=v^2_{(u',w')}$. Let $w$ be such that $G\models P_3(w)\land  E(u',u)\land E(u,w)\land E(w,w')$ and $w''$ be such that $G\models P_3(w'')\land E(u',u'')\land E(u'',w'')\land E(w'',w')$. Plugging $(x,x')=(u',w')$ to $\vp_3$ which is satisfied by $G$, by uniqueness, we obtain $(u,w)=(u'',w'')$, so also $u=u''$, as required. The case $G\models P_3(u)$ is the same as the case $G\models P_2(u)$. So $f$ is indeed injective.

To show that $f$ is surjective, suppose $v$ is a node of $\Gad(G')$. There are three cases to consider. If $v=v^1_u$ with $u\in V'$, then $f(u)=v$. If $v=v^2_{(u',w')}$ with $u',w'\in V'$, then $(u',w')\in E'$, so there is a gadgetised edge from $u'$ to $v'$ in $G$, in particular, there are a nodes $u,w\in V$ with $G\models P_2(u)$, $G\models P_3(w)$ and $G\models  E(u',u)\land E(u,w)\land E(w,w')$. Moreover, as $G$ satisfies $\vp_3$, given $u,w$, the nodes $u',w'$ for which this is satisfied are unique. Thus $f(u)=v^2_{(u',w')}$, as required. Finally, the case $v=v^3_{(u,w)}$ is the same as the case $v=v^2_{(u,w)}$, so $f$ is indeed surjective.

We will now show that $(u,v)\in E$ if and only if $(f(u),f(v))$ is an edge in $\Gad(G')$.

Assume $(u,v)\in E$. As $G$ satisfies $\vp_2$, there are three cases to consider: $G\models P_1(u)\land P_2(v)$, $G\models P_2(u)\land P_3(v)$ and $G\models P_3(u)\land P_1(v)$. If $G\models P_1(u)\land P_2(v)$, then $f(u)=v_u^1$ and $f(v)=v^2_{(u',w')}$ for the unique $u',w'\in V'$ with $G\models  \exists w \big( P_3(w)\land E(u',u)\land E(u,w)\land E(w,w') \big)$. Note that $(u,v)\in E$, so by uniqueness of $u'$, we have $u=u'$, so $(f(u),f(v))=(v_u^1,v^2_{(u,w')})$ is an edge in $\Gad(G')$. If $G\models P_2(u)\land P_3(v)$, then $f(u)=v^2_{(u',w')}$ and $f(v)=v^3_{(u'',w'')}$, again similarly as before, as $(u,v)\in E$, we obtain by uniqueness, that $(u',w')=(u'',w'')$, so $(f(u),f(v))=(v^2_{(u',w')},v^3_{(u',w')})$ is an edge in $\Gad(G')$. If $G\models P_3(u)\land P_1(v)$, then we proceed as in the first case.

Conversely, assume $(f(u),f(v))$ is an edge in $\Gad(G')$. By construction of gadgetisation there are three cases to consider: $f(u)=v^1_{u'}$ and $f(v)=v^2_{(u',w')}$, $f(u)=v^2_{(u',w')}$ and $f(v)=v^3_{(u',w')}$ and finally $f(u)=v^3_{(u',w')}$ and $f(w)=v^1_{w'}$. If $f(u)=v^1_{u'}$ and $f(v)=v^2_{(u',w')}$, then $u=u'$ and $G\models \exists w(P_3(w)\land E(u',v)\land E(v,w)\land E(w,w')$, in particular $(u',v)\in E$, but $u=u'$, so $(u,v)\in E$, as required. If $f(u)=v^2_{(u',w')}$ and $f(v)=v^3_{(u',w')}$, then for some $w_1\in V$, $G\models (P_3(w_1)\land E(u',u)\land E(u,w_1)\land E(w_1,w')$ and for some $w_2\in V$, $G\models (P_2(w_2)\land E(u',w_2)\land E(w_2,v)\land E(v,w')$, so by uniqueness of the gadgetised edge from $u'$ to $w'$, we obtain $(u,w_1)=(w_2,v)$, in particular, as $(u,w_1)\in E$, we also have $(u,v)\in E$, as required. If $f(u)=v^3_{(u,w)}$ and $f(w)=v^1_w$, then we proceed as in the first case.

Finally, by definition of labels of a gadgetisation we get that $\lambda(v^1_u)=(1,0,0),\ \lambda(v^2_{(u',v')})=(0,1,0)$ and $\lambda(v^3_{(u',v')})=(0,0,1)$, so in particular $u\in P_i$ if and only if $f(u)\in P_i$, for $i\in \{1,2,3\}$.
\end{proof}

\GLOisGNN*

\begin{proof}
We will show that for any  undirected graph $G$ and a node $v$ we have $G\models \vp_{\GLO}(v)$ if and only if  $G \models \vp_1 \land \vp_2$ (for $\vp_1$ and $\vp_2$ from the proof of \Cref{gadlinisfo}) and $G$ satisfies an additional property $\psi$ (which we specify below).

First define $N_2(v):=|\{w:G\models P_2(w)\land E(v,w)\}|$ and consider the following property. \begin{align*} \psi =& \text{ For all integers $i,j$ with $|P_1|>j > i$: there exist node $u,w \in V$ with $G\models P_1(u)\land P_1(w)$, $N_2(u)=j,N_2(w)=i$}\\ &\text{ and }\exists u'\exists w' P_2(u')\land P_2(w')\land E(u,u')\land E(u',w')\land E(w',w),
\end{align*}
    
We claim  that over undirected graphs, $\vp_\GLO(x)$ is equivalent to $\vp_1 \land \vp_2 \land \psi \land x=x$.
For the forward implication assume that $G' \models \vp_\GLO(v)$, so $G' = \Gad(G)$, for some strict linear order $G$.
Directly by the definition of $\Gad(G)$, we obtain that $\Gad(G) \models \vp_1$ and 
$\Gad(G) \models \vp_2$. Since $G$ is a strict linear order, its set of out degrees is precisely $\{0,1,\ldots |G|-1\}$ with the directed edge between nodes with out degrees $j$ and $i$ respectively, where $j>i$, goes from $j$ to $i$. Thus $G'$ satisfies $ \psi$.\\

 For the opposite direction assume that an  undirected graph $G=(V,E,\lambda)$ satisfies $\vp_1$, $\vp_2$ and $\psi$. We will show that $G$ satisfies $\vp_3$ and $\vp_4$, so by Theorem \ref{gadlinisfo}, $G$ satisfies $\vp_\GLO$.

Suppose $G\models \neg \vp_3$. As $G$ satisfies $\vp_1$, $P_1,P_2$ and $P_3$ form a partition of the set $V$. As $G$ satisfies $\vp_2$, for every node $u\in V$ with $G\models P_2(u)$, there is a unique node $w\in V$ with $G\models P_1(w)\land E(u,w)$, so by the double counting principle, we get  $$|P_2|=|\{\{u,v\}:G\models P_1(u)\land E(u,w)\land P_2(w)\}|=\sum_{u\in P_1}N_2(u).$$

Since $G \models \psi$, for each $i \in \{0, 1, \ldots, |P_1| - 1\}$, there exists a node $u \in V$ such that $G \models P_1(u)$ and $N_2(u) = i$. As there are exactly $|P_1|$ nodes $u \in V$ with $G \models P_1(u)$, it follows that for each such $i$, there is a unique node $u \in V$ with $G \models P_1(u)$ and $N_2(u) = i$. Moreover, these are the only nodes satisfying $G \models P_1(u)$.

Finally, combining the above results with the fact that $G$ satisfies $\psi$, we get that for any $u,w\in V$ with $G\models P_1(u)\land P_1(w)$, there is a gadgetised edge from $u$ to $w$ or from $w$ to $u$. In the light of $G\models \neg \vp_3,$ it must be the case, that there is at least one pair of nodes $u,w\in V$ with $G\models P_1(u)\land P_1(w)$, where there is a gadgetised edge from $u$ to $w$ and from $w$ to $u$. But then $$|P_2|\geqslant {{|P_1|}\choose{2}}+1>\sum_{i=0}^{|P_1|-1} i=\sum_{u\in P_1}N_2(u)=|P_2|,$$ which is a contradiction.

Suppose $G\models\neg \vp_4$. Exactly as in the proof of Proposition \ref{linear}, we can show by induction that for $u,v\in V$ with $G\models P_1(u)\land P_1(v)$, $u$ is connected to $v$ with a gadgetised edge if and only if $N_2(u)>N_2(v)$. As $G\models\neg \vp_4$, there are nodes $u,v,w\in V$ with $G\models P_1(u)\land P_1(v)\land P_1(w)$, $u$ is connected to $v$, $v$ is connected to $w$ and $w$ is connected to $u$ with a gadgetised edge. Thus, by previous, $N_2(u)>N_2(v)>N_2(w)>N_2(u)$, which is a contradiction.\\

We now show how to construct, for $\vp_1,\vp_2$ and $\psi$, an ACR-GNN that computes it. This suffices, because once we know that each of those properties is captured by some ACR-GNN, we can construct a single ACR-GNN that captures their conjunction. The construction works as follows: at each layer, the new feature vector at a node is defined by concatenating the feature vectors produced at that stage by three independent ACR-GNNs, each computing one of $\vp_1,\vp_2$ and $\psi$. The aggregate, combine, and readout functions are likewise modular: each operates independently on the segment of the concatenated vector corresponding to the respective formula. Finally, the output classifier accepts a graph if and only if all three component classifiers, applied to their respective segments of the final concatenated vector, also accept. Since this construction fits within the architectural rules of an ACR-GNN, it follows that the conjunction of $\vp_1,\vp_2$ and $\psi$ is also definable by an ACR-GNN.

We now show how to construct, for $\vp_1,\vp_2$ and $\psi$, an ACR-GNN that computes it. Note that in the proof of Theorem \ref{gadlinisfo}, $\vp_1$ and $\vp_2$ are represented as $\FO^2$ formulas. Thus by Theorem 5.1 in \cite{DBLP:conf/iclr/BarceloKM0RS20}, we get that $\vp_1$ and $\vp_2$ are captured by an ACR-GNN.

We now construct an ACR-GNN $\N$ that captures $\psi$. Recall that  gadgetised linear orders are graphs of dimension three, so we will consider application of $\N$ to such graphs $G$.
In each layer,  $\N$ will assign to nodes vectors of dimesion five, where the first three positions are always as in the input graph $G$, so  information about  $P_1$, $P_2$, and $P_3$ in the input graph is preserved across all layers. 
The fourth and fifth positions will always keep binary numbers.
The details of $\N$ are provided next.

Let $\A$ denote the aggregate which maps a multiset $M$ to the value $$\A(M):=\bigvee_{m\in M} m,$$ where $\bigvee$ stands for the bitwise OR on number in decimal representation, i.e. for example $100\bigvee 100\bigvee 110=110$. Moreover, we write $\A_{m\in M}(f(m))$ to stand for the aggregate $\bigvee_{m\in M} f(m)$.

The first layer computes, for each node $u \in V$ with $G \models P_1(u)$, the value $N_2(u)$ and stores it in $x(u)$. This is achieved with an AC layer, with $\agg(M)=10^{sum(M)}$ and $\comb(w,v)$ to be the identity on $w$ if $w_1=0$ and the function that changes the value of $w_4$ to $v_2$ otherwise.

The second layer, for each node $u\in V$ with $G\models P_3(u)$, computes and stores in $y(u)$ the value of $$\A(\{{x(v)}:G\models P_1(v)\land E(u,v)\}).$$ This is achieved with an AC layer, with $\agg(M)=\A_{w\in M}(w_1{w_4})$ and $\comb(w,n)$ to be the identity on $w$ if $w_3=0$ and the function that changes the value of $w_5$ to $n$ otherwise.

The third layer, for each node $u\in V$ with $G\models P_2(u)$, computes and stores in $y(u)$ the value of $$\A(\{y(v):G\models P_3(v)\land E(u,v)\}).$$ This is achieved with an AC layer, with $\agg(M)=\A_{w\in M}(w_3{w_5})$ and $\comb(w,n)$ to be the identity on $w$ if $w_2=0$ and the function that changes the value of $w_5$ to $n$ otherwise.

Finally, the fifth 
layer uses a global readout to assign 1 to each node if 
for all $i<j<|P_1|$ there exists a node whose fourth position of the vector is $10^j$ and the fifth position of the vector has $1$ as the $i$th bit from the right (when counting from 0).

The  first four layers can be implemented without readout functions.
The fifth layer, in contrast, requires using readout, but no aggregation.
To show that the construction is correct, we
can show that in layer 4, each node $v$ satisfying $P_1$ has on the fourth position of its vector  $10^j$, where $j$ is the number of $v$ neighbours satisfying $P_2$.
On the fifth position $v$ has a binary number, whose $i$th bit is $1$ if there is a gadgetised edge from $v$ to some node with $i$ neighbours satisfying $P_2$.
Therefore, the fifth layer assigns 1 to all nodes if the graphs satisfies $\psi$, and otherwise it assigns $0$ to all nodes.

The fourth layer, for each node $u\in V$ with $G\models P_1(u)$, computes and stores in $y(u)$ the value of $$\A (\{y(v):G\models P_2(v)\land E(u,v)\}).$$ This is achieved with an AC layer, with $\agg(M)=\A_{w\in M}(w_2{w_5})$ and $\comb(w,n)$ to be the identity on $w$ if $w_1=0$ and the function that changes the value of $w_5$ to $n$ otherwise. 

Finally, the fifth 
layer uses a global readout to assign 1 to each node if 
for all $i<j<|P_1|$ there exists a node whose fourth position of the vector is $10^j$ and the fifth position of the vector has $1$ as the $i$th bit from the right (when counting from 0).\\

To show that the construction is correct, note that in layer 4, each node $v$ satisfying $P_1$ has on the fourth position of its vector  $10^j$, where $j$ is the number of $v$ neighbours satisfying $P_2$ and  on the fifth position $v$ has a binary number, whose $i$th bit is $1$ if there is a gadgetised edge from $v$ to some node with $i$ neighbours satisfying $P_2$.
Therefore, the fifth layer assigns 1 to all nodes if the graphs satisfies $\psi$, and otherwise it assigns $0$ to all nodes.

\end{proof}

\GLonotC*
\begin{proof}
Suppose towards a  contradiction that $\varphi_{\GLO}(x)$ is expressible in $\Ctwo$, so it is definable by a formula in $C^2_{\ell, c}$, for some $\ell$ and $c$.
To obtain a contradiction, we will construct a graph $H$ with a node $v$ and a graph $H'$ with a corresponding node $v'$ such that $H,v \models \varphi_{\GLO}(v)$ and $H',v' \not\models \varphi_{\GLO}(v')$, but $H,v \equiv_{\Ctwo_{l,c}} H',v'$ for any node $v$ and its corresponding node $v'$.

Let $n = \ell \cdot c +1$. We define $H=\Gad(G)$, where $G$ is the same as defined in the proof of Theorem \ref{thm:LonotC2}. Comment on notation: we will write $v_i^1$, instead of $v^1_{v_i}$ and $v_{i,j}^\alpha$ instead of $v_{(v_i,v_j)}^\alpha$ for $\alpha\in \{2,3\}$.

To define $H'$, let $V(H'):=\{v':v\in V(H)\}$ and let $E(H')$ be the primed version of the set below
\begin{align*}[E(H)\cup\{(v^1_{1},v^3_{1,-1}),(v^3_{1,-1},v^2_{1,-1}),(v^2_{1,-1},v^1_{-1})] \setminus (v^1_{1},v^2_{1,-1}),(v^2_{1,-1},v^3_{1,-1}),(v^3_{1,-1},v^1_{-1}),\end{align*}

Clearly, \( H \models \varphi_{\GLO}(v) \) for every node \( v \). However, note that \( H' \) contains a cycle of length 9, namely ${v^1_{1}}'-{v^2_{1,0}}'-{v^3_{1,0}}'-{v^1_0}'- {v^2_{0,-1}}'-{v^3_{0,-1}}'-{v^1_{-1}}'-{v^2_{1,-1}}'-{v^3_{1,-1}}'-{v^1_{1}}'$, where the nodes have cyclic labels \( (1,0,0), (0,1,0), (0,0,1) \).  
Note that no graph satisfying $\vp_{\GLO}(x)$ contains such a substructure, so \( H' \not\models \varphi_{\GLO}(v') \) for any node \( v' \).

Thus it remains to show that $H,v \equiv_{\Ctwo_{l,c}} H',v'$ for all nodes $v,v'$. To this end, by \Cref{mainWLtheorem},
 it suffices to show that $W^\ell_c (v) = W^\ell_c (v')$.
We will prove it showing, by simultaneous induction on $k \leq \ell$,  the  following  statements 
\begin{enumerate}[label=(\roman*), leftmargin=*, align=left, labelsep=-0.5em]
\item $W^k_c (v) = W^k_c (v')$, for $v\in V(H)$.

\item $W^k_c (v^1_i) = W^k_c (v^1_j)$, for $i,j\in  \{ -(n - ck), \dots,  n -ck   \}$


\item $W^k_c (v^2_{i,a}) = W^k_c (v^2_{j,b})$, $W^k_c (v^3_{i,a}) = W^k_c (v^3_{j,b})$ for $i>a,j>b\in \{ -(n - ck), \dots,  n -ck   \}.$

\end{enumerate}

In the base of the induction, for $k=0$, Statements~(i) - (iii) hold, since $W^0_c (v_i^1) = W^0_c ({v^1_i}')=(0,0,1)$, $W^0_c (v^2_{i,j}) = W^0_c (v^2_{i,j})=(0,1,0)$ and $W^0_c (v^3_{i,j}) = W^0_c (v^3_{i,j})=(0,0,1)$ for all $i,j\in \{-n,\ldots, n\}.$

For the inductive step, assume that Statements (i) - (iii) hold for some $k<\ell$. We will show that they hold for $k+1$.

We start by showing Statement (ii). 
Let us fix any $i,j \in \{ -(n - c(k+1)), \dots,  n -c(k+1)   \}$.
Since $E$ and $E'$ are irreflexive and symmetric, to prove that that $W^{k+1}_c (v_i^1) = W^{k+1}_c ({v_j^1})$, it suffices to show the following equalities:
\begin{itemize}
\item[(1)] 
$W^{k}_c (v_i^1) = W^{k}_c ({v_j}^1)$, 
\item[(2)] 
$\lBrace W^{k}_c (v) : v \in N_G(v_i^1)  \rBrace^c = \lBrace W^{k}_c (v) : v \in N_G({v_j^1})  \rBrace^c $,
\item[(3)] 
$\lBrace W^{k}_c (v) : v_i^1\neq v \notin N_G(v_i^1)  \rBrace^c = \lBrace W^{k}_c (v) : {v_j^1}\neq v \in N_G({v_j}^1)  \rBrace^c $.
\end{itemize}

Equality (1) holds by the inductive assumption for Statement (ii).
To show Equality (2), as no two nodes $v_s^1$ are connected, it suffices to show that 

\begin{itemize}
\item[$(A)$] 
$\lBrace W^{k}_c (v^2_{s,t}) : v^2_{s,t} \in N_G(v_i^1)  \rBrace^c = \lBrace W^{k}_c ({v^2_{s,t}}) : {v^2_{s,t}} \in N_G({v_j^1})  \rBrace^c $,
\item[$(B)$] 
$\lBrace W^{k}_c (v^3_{s,t}) : v^3_{s,t} \in N_G(v_i^1)  \rBrace^c = \lBrace W^{k}_c ({v^3_{s,t}}) : {v^3_{s,t}} \in N_G({v_j^1})  \rBrace^c.$
\end{itemize}

We start by showing equality $(A)$. We let
 $S = \{ -(n-ck), \dots, n-ck  \}$. 
We will show two versions of this equality, for multisets with  $\neg (s,t \in S)$ and with $s,t \in S$ (which is stronger  than original equalities with all $s,t$).
For $\neg(s,t \in S)$, we have 
$v^2_{s,t} \in N_G(v_i^1)$ iff ${v^2_{s,t}} \in N_G({v^1_j})$, so equality $(A)$ holds. 
Now consider multisets with $s,t \in S$. By the inductive assumption for Statement (iii),
$W^{k}_c (v^2_{s,t})$ is the same for all $s,t \in S$.
So to prove Equality $(A)$, it suffices to show that both $v_i^1$ and $v_j^1$ have at least $c$ many neighbours $v_{s,t}^2$ with $s,t \in S$.
For this, recall that 
$i,j \geqslant - (n-c(k+1))$,
so for each $r\in \{ - (n-ck) , \dots ,  -(n-c(k+1))-1 \}$ we have $i,j>r$, so both $(v_i^1,v_{i,r}^2) \in E$ and $(v_j^1,v_{j,r}^2) \in E$.
Note that there are exactly $c$ such indices $t$, which finishes the proof of Equality $(A)$.

Equality $(B)$ is proved exactly in the same way as $(A)$, so (2) also follows.

To show Equality (3), it suffices to show that 

\begin{itemize}
\item[$(A')$] 
$\lBrace W^{k}_c (v^2_{s,t}) : v^2_{s,t} \notin N_G(v_i^1)  \rBrace^c = \lBrace W^{k}_c ({v^2_{s,t}}) : {v^2_{s,t}} \notin N_G({v_j^1})  \rBrace^c $,
\item[$(B')$] 
$\lBrace W^{k}_c (v^3_{s,t}) : v^3_{s,t} \notin N_G(v_i^1)  \rBrace^c = \lBrace W^{k}_c ({v^3_{s,t}}) : {v^3_{s,t}} \notin N_G({v_j^1})  \rBrace^c.$
\item[$(C')$] 
$\lBrace W^{k}_c (v^1_{t}) :v^1_i\neq v^1_{t} \notin N_G(v_i^1)  \rBrace^c = \lBrace W^{k}_c ({v^1_{t}}) : v^1_j\neq {v^1_{t}} \notin N_G({v_j^1})  \rBrace^c.$
\end{itemize}

We start by showing equality $(A')$. 
We will again show two versions of this equality. For multisets with  $\neg (s,t \in S)$ we proceed as in the case of $(A)$. For multisets with $s,t \in S$, by the inductive assumption for Statement (iii),
$W^{k}_c (v^2_{s,t})$ is the same for all $s,t \in S$.
So to prove Equality $(A')$, it suffices to show that both $v_i^1$ and $v_j^1$ have at least $c$ many non-neighbours $v_{s,t}^2$ with $s,t \in S$.
For this, recall that 
$i,j \leqslant n-c(k+1)$,
so for each $r\in \{ n-c(k+1)+1,\ldots, n-ck \}$ we have $r>i,j$, so both $\{v_i^1,v_{i,r}^2\} \notin E$ and $\{v_j^1,v_{j,r}^2\} \notin E$.
Note that there are exactly $c$ such indices $r$, which finishes the proof of Equality $(A')$.

Equality $(B')$ is proved exactly in the same way as $(A')$, Equality~$(C')$ follows because $W^k_c(v_i)=W^k_c(v_j)$. Thus Equality~(3) holds, so in consequence Equality~(ii) holds.\\

To show Statement (iii). Let us fix any $i>a,j>b \in \{ -(n - c(k+1)), \dots,  n -c(k+1)   \}$. To prove that $W^{k+1}_c (v^2_{i,a}) = W^{k+1}_c ({v^2_{j,b}})$ it satisfies to show that $W^{k}_c (v^2_{i,a}) = W^{k}_c ({v^2_{j,b}})$, $W^{k}_c (v^1_{i}) = W^{k}_c ({v^1_j})$ and $W^{k}_c (v^3_{i,a}) = W^{k}_c ({v^3_{i,a}})$, which all follows immediately from the induction hypothesis of (ii) and (iii). Showing $W^{k+1}_c (v^3_{i,a}) = W^{k+1}_c ({v^3_{j,b}})$ is completely analogous, so Equality (iii) holds.\\

Next, we show the inductive step for Statement (i).
We start by observing that $W^{k+1}_c (v_i^1) = W^{k+1}_c ({v_i^1}')$ for $i \notin \{-1, 1 \}$ and $W^{k+1}_c (v_{a,b}^2) = W^{k+1}_c ({v_{a,b}^2}')$, $W^{k+1}_c (v_{a,b}^3) = W^{k+1}_c ({v_{a,b}^3}')$ for $(a,b)\neq (1,-1)$ follows from the inductive assumption for  Statement (i) together with fact that the nodes under consideration have the same $E$-neighbours (modulo priming of symbols).
It remains to show Statement (i) for $i \in \{-1,1 \}$ and $(a,b)=(1,-1)$.
Note that we have $W^{k+1}_c (v_0^1) = W^{k+1}_c ({v_0^1}')$.
By the inductive step for Statement (ii),
we obtain that $W^{k+1}_c (v_{-1}^1) = W^{k+1}_c (v_0^1) = W^{k+1}_c (v_{1}^1)$.
Although we have showed Statement (ii) for $H$, the same argumentation can be used for $H'$,  so  $W^{k+1}_c ({v_{-1}^1}') = W^{k+1}_c ({v_0^1}') = W^{k+1}_c ({v_{1}^1}')$.
Thus, 
$W^{k+1}_c (v_{i}^1) = W^{k+1}_c ({v_{i}^1}')$  for $i \in \{-1, 1 \}$. Finally consider $(a,b)=(1,-1)$. To prove that $W^{k+1}_c (v^2_{1,-1}) = W^{k+1}_c ({v^2_{1,-1}}')$ it satisfies to show that $W^{k}_c (v^2_{1,-1}) = W^{k}_c ({v^2_{1,-1}}')$, $W^{k}_c (v^1_{1}) = W^{k}_c ({v^1_{-1}}')$ and $W^{k}_c (v^3_{1,-1}) = W^{k}_c ({v^3_{1,-1}}')$, which all follows immediately from the induction hypothesis of (i) and (ii). Showing $W^{k+1}_c (v^3_{1,-1}) = W^{k+1}_c ({v^3_{1,-1}})$ is completely analogous, so Equality (i) holds.
 \end{proof}